\documentclass[11pt]{article}
\usepackage[utf8]{inputenc}

\def\Comments{0}  

\newcommand{\citep}[1]{\cite{#1}}
\newcommand{\citet}[1]{\cite{#1}}

\usepackage{hyperref}
\usepackage{amsmath,amsthm,amssymb,caption,subcaption,cleveref,nicefrac}
\usepackage{fullpage}
\usepackage{bm}
\usepackage[table]{xcolor}
\usepackage{multirow}
\usepackage{arydshln}
\usepackage{csquotes}
\usepackage{anyfontsize}
\usepackage[shortlabels]{enumitem}

\usepackage{natbib}

\usepackage{algorithm}
\usepackage{algorithmic}
\usepackage[algo2e,ruled,noend,vlined]{algorithm2e}

\usepackage{bm,dsfont}
\usepackage{multirow}
\usepackage{booktabs}
\usepackage{nicefrac}

\usepackage{enumitem}
\setitemize[1]{leftmargin=*,itemsep=0pt}

\def\ddefloop#1{\ifx\ddefloop#1\else\ddef{#1}\expandafter\ddefloop\fi}
\def\ddef#1{\expandafter\def\csname #1\endcsname{\ensuremath{\mathbb{#1}}}}
\ddefloop ABCDFGHIJKLMNOQRTUVWXYZ\ddefloop  

\def\ddef#1{\expandafter\def\csname c#1\endcsname{\ensuremath{\mathcal{#1}}}}
\ddefloop ABCDEFGHIJKLMNOPQRSTUVWXYZ\ddefloop  
\def\ddef#1{\expandafter\def\csname b#1\endcsname{\ensuremath{\bm #1}}}
\ddefloop ABCDEFGHIJKLMNOPQRSTUVWXYZabcdeghijklnopqrstuvwxyz\ddefloop  
\def\ddef#1{\expandafter\def\csname h#1\endcsname{\ensuremath{\hat{#1}}}}
\ddefloop Ncefghnu \ddefloop  
\def\ddef#1{\expandafter\def\csname t#1\endcsname{\ensuremath{\tilde{#1}}}}
\ddefloop GOefguw \ddefloop  
\def\ddef#1{\expandafter\def\csname tb#1\endcsname{\ensuremath{\tilde{\bm #1}}}}
\ddefloop ghwz \ddefloop  

\newcommand{\prn}[1]{\left ( #1 \right )}
\newcommand{\ang}[1]{\left \langle #1 \right \rangle}
\newcommand{\sq}[1]{\left [ #1 \right ]}

\definecolor{Gred}{RGB}{219, 50, 54}
\definecolor{Ggreen}{RGB}{60, 186, 84}
\definecolor{Gblue}{RGB}{72, 133, 237}
\definecolor{Gyellow}{RGB}{247, 178, 16}
\definecolor{ToCgreen}{RGB}{0, 128, 0}
\definecolor{myGold}{RGB}{231,141,20}
\definecolor{myBlue}{rgb}{0.19,0.41,.65}
\definecolor{myPurple}{RGB}{175,0,124}

\providecommand{\Comments}{0}
\ifnum\Comments>0
    \usepackage[colorinlistoftodos,prependcaption,textsize=scriptsize]{todonotes}
    \paperwidth=\dimexpr \paperwidth + 2cm\relax
    \oddsidemargin=\dimexpr\oddsidemargin + 1cm\relax
    \evensidemargin=\dimexpr\evensidemargin + 1cm\relax
    \marginparwidth=\dimexpr\marginparwidth + 3.6cm\relax
\else
    \usepackage[disable]{todonotes}
\fi
\newcommand{\mytodo}[1]{\ifnum\Comments<2{#1}\fi}
\newcommand{\mytodoTwo}[1]{\ifnum\Comments<3{#1}\fi}
\newcommand{\mytodoThree}[1]{\ifnum\Comments<4{#1}\fi}

\newcommand{\todoinline}[1]{\ifnum\Comments<4\todo[inline,linecolor=Gred,backgroundcolor=Gred!25,bordercolor=Gred]{#1}\fi}

\newcommand{\tableoftodos}{\ifnum\Comments=1 \listoftodos[Comments/To Do's] \fi}

\newcommand{\eps}{\varepsilon}

\newcommand{\SVT}{\mathrm{SVT}}

\newcommand{\PREM}{\mathsf{PREM}}

\renewcommand{\phi}{\varphi}

\newcommand{\tOmega}{\tilde{\Omega}}

\DeclareMathOperator{\Lap}{Lap}

\DeclareMathOperator{\acti}{active}

\DeclareMathOperator{\equal}{\mbox{\sc Approx}}
\DeclareMathOperator{\plus}{+1}
\DeclareMathOperator{\minus}{--1}

\newcommand{\hbh}{\hat{\bh}}
\newcommand{\counter}{\mbox{counter}}

\newcommand{\OutsideIntervalMonitor}{\textsc{RangeMonitor}}
\newcommand{\OIM}{\mathsf{RM}}
\newcommand{\FindMarginExample}{\textsc{FindMarginExample}}
\newcommand{\Above}{\textsc{Above}}
\newcommand{\Below}{\textsc{Below}}
\newcommand{\Between}{\textsc{Inside}}

\newcommand{\ind}{\mathds{1}} 

\newcommand{\Factive}{\cF_{\acti}}
\newcommand{\Sactive}{\cS_{\acti}}
\newcommand{\Splus}{\cS^{+1}}
\newcommand{\Sminus}{\cS^{-1}}

\newcommand{\algcomment}[1]{\hfill\textcolor{black!50}{$\ldots$ #1}}

\renewcommand{\setminus}{\smallsetminus}

\newtheorem{theorem}{Theorem}[section]
\newtheorem{lemma}[theorem]{Lemma}

\newtheorem{proposition}[theorem]{Proposition}
\newtheorem{claim}[theorem]{Claim}

\theoremstyle{definition}

\newtheorem{definition}[theorem]{Definition}

\newenvironment{proofof}[1]%
{%
\par\noindent{\bfseries\upshape \proofname\xspace of #1.\ }%
}%
{\qed}

\title{PREM: Privately Answering Statistical Queries with Relative Error}
\author{Badih Ghazi\\Google Research\\{\small \texttt{badihghazi@gmail.com}}
\and
Crist\'obal Guzm\'an\\Google Research and \\Pontificia Universidad Cat\'olica de Chile\\{\small \texttt{crguzman@google.com}}
\and
Pritish Kamath\\Google Research\\{\small \texttt{pritish@alum.mit.edu}}
\and
Alexander Knop\\Google Research\\{\small \texttt{alexanderknop@google.com}}
\and
Ravi Kumar\\Google Research\\{\small \texttt{ravi.k53@gmail.com}}
\and
Pasin Manurangsi\\Google Research\\{\small \texttt{pasin@google.com}}
\and
Sushant Sachdeva\\Google Research and \\University of Toronto\\{\small \texttt{susachdeva@google.com}}
}
\date{\today}

\begin{document}

\maketitle

\begin{abstract}
We introduce $\mathsf{PREM}$ (Private Relative Error Multiplicative weight update), a new framework for generating synthetic data that achieves a {\em relative} error guarantee for statistical queries under $(\varepsilon, \delta)$ differential privacy (DP).
Namely, for a domain ${\cal X}$, a family ${\cal F}$ of queries $f : {\cal X} \to \{0, 1\}$, and $\zeta > 0$, our framework yields a mechanism that on input dataset $D \in {\cal X}^n$ outputs a synthetic dataset $\widehat{D} \in {\cal X}^n$ such that all statistical queries in ${\cal F}$ on $D$, namely $\sum_{x \in D} f(x)$ for $f \in {\cal F}$, are within a $1 \pm \zeta$ {\em multiplicative} factor of the corresponding value on $\widehat{D}$ up to an {\em additive error} that is polynomial in $\log |{\cal F}|$, $\log |{\cal X}|$, $\log n$, $\log(1/\delta)$, $1/\varepsilon$, and $1/\zeta$.
In contrast, any $(\varepsilon, \delta)$-DP mechanism is known to require worst-case additive error that is polynomial in at least one of $n, |{\cal F}|$, or $|{\cal X}|$.
We complement our algorithm with nearly matching lower bounds.
\end{abstract}

\begin{small}
\end{small}



\section{Introduction}\label{sec:intro}
Differential Privacy (DP)~\citep{Dwork:2006} has become the de facto standard for privacy preserving data analysis. An important use case is that of releasing statistics about sub-populations within a dataset while protecting privacy of individual records within it. DP has seen several practical deployments in the recent years, e.g., the US census being a significant one~\citep{abowd22topdown}.

Formally, consider a dataset $D = (x_1, \ldots, x_n) \in \cX^n$ of \emph{records} $x_i \in \cX$ for some finite \emph{domain} $\cX$, and consider a \emph{query} of interest that is a function $f : \cX \to \{0, 1\}$. We are interested in privately releasing the evaluation of the statistical query (a.k.a. linear query) $f \in \cF$ on dataset $D$, namely $f(D) := \sum_{x \in D} f(x)$, where $\cF$ is a family of queries. To motivate this definition, consider a dataset where each record $x$ contains several attributes about a unique individual in a population, such as their age, gender, race, zip-code, education level, income, etc. A potential query function $f$ of interest could be such that $f(x) = 1$ if the record $x$ corresponds to a male person of age in 30-35, with income above \$50,000 in all zip-codes within a certain county, and $f(x)=0$ otherwise.

Mechanisms that satisfy DP provide randomized estimates $(\he_f)_{f \in \cF}$ of these statistical query evaluations $(f(D))_{f \in \cF}$. The quality of these estimates are measured in terms of their accuracy, with a typical approach being to measure the {\em additive error}, namely, we would say that the estimates $\hat{\be} \in \R^{\cF}$ is $(\alpha, \cF)$-accurate with respect to dataset $D$ if $|\hat e_f - f(D)| \le \alpha$ for all $f \in \cF$. Prior work has introduced numerous mechanisms for releasing statistical query evaluations, which we summarize in \Cref{tab:prior-work-our-results}; while some of these mechanisms only release the estimates $\hat{\be} \in \R^{\cF}$, others generate a synthetic dataset $\widehat{D}$ from which the estimates $\he_f := f(\widehat{D})$ can be derived for any $f \in \cF$.

These results are stated informally with $\tO(\cdot)$ notation that hides lower order polylogarithmic terms\footnote{Formally, $\tO(f)$ means a quantity that is $O(f \log^c f)$ for some constant $c$, where the asymptotics are in terms of $n, |\cX|, |\cF|, 1/\zeta, 1/\eps, 1/\delta$.} 
wherever applicable, and we also consider a small, but constant, failure probability, e.g., $0.01$. An important point to note about all the additive error bounds are that they incur an error that is polynomial in either $n$, $|\cF|$, or $|\cX|$. Moreover, such dependencies are known to be required for {\em additive-only} error.

{
\begin{table}[t]\renewcommand{\arraystretch}{1.8}
    \centering\fontsize{9pt}{10}\selectfont
    \begin{tabular}{|c|c|c|c|c|}
    \hline
    {\bf Reference} & {\bf Error type} & {\bf Error bound} & {\bf DP} & {\bf Release} \\
    \hline
     \cite{steinke16between} & \multirow{5}{*}{Additive} & $O\prn{\frac{|\cF|}{\eps}}$ & $\eps$-DP & \multirow{2}{*}{Estimates}\\
     \cline{1-1}\cline{3-4}
    \raisebox{-5pt}{\footnotesize \shortstack{\cite{DK22}\\ \cite{GKM21}}} & & $O\prn{\frac{\sqrt{|\cF| \log \frac 1\delta}}{\eps}}$ & $(\eps, \delta)$-DP & \\
    \cline{1-1}\cline{3-5}
    {\footnotesize \cite{vadhan17complexity}} &  & $O\prn{\frac{\sqrt{|\cX| \log |\cF|}}{\eps}}$ & $\eps$-DP & \multirow{3}{*}{Syn. Data} \\
    \cline{1-1}\cline{3-4}
    \raisebox{-7pt}{\footnotesize \shortstack{\cite{BlumLR13}\\\cite{HardtR10}\\\cite{HardtLM12}}} & & $O\prn{n^{\frac23} \prn{\frac{\log |\cF| \log|\cX|}{\eps}}^{\frac13}}$ & $\eps$-DP & \\
    \cline{1-1}\cline{3-4}
    \raisebox{-5pt}{\footnotesize \shortstack{\cite{HardtR10}\\\cite{HardtLM12}}} & & $O\prn{n^{\frac12}\prn{\frac{\sqrt{\log |\cX|} \log |\cF| \log \frac1\delta}{\eps}}^{\frac12}}$ & $(\eps, \delta)$-DP & \\
    \hline
    \hline
    \Cref{thm:efficient_pure_DP_UP} (This work) & \multirow{2}{*}{Relative} & $\tO\prn{\sqrt{\frac{n}{\eps \zeta^2}\log |{\cF}|\log|{\cX}|}+\frac{1}{\eps}}$ & $\eps$-DP & \multirow{2}{*}{Syn. Data} \\
    \Cref{thm:main} (This work) & & $\tO\prn{\frac{1}{\zeta\eps} \prn{\log n \log \frac1\delta}^{\frac32} \sqrt{\log |\cX|}\log|\cF|}$ & $(\eps, \delta)$-DP & \\
    \hline
    \end{tabular}
    \caption{Additive error bounds in prior work \& relative error bounds in this work.}
    \label{tab:prior-work-our-results}
\end{table}
}

The additive nature of the error bound can however be limiting in some settings. Consider a statistical query $f(D)$ that counts the number of people matching a certain rare combination of attributes. In such cases, the additive error can completely overwhelm the true value of the statistical query $f(D)$. On the other hand, for a statistical query $f(D)$ that counts the number of records matching a common combination of attributes, the additive error will be much smaller compared to $f(D)$.%
In such cases, it can be more meaningful to have a notion of {\em relative} error, wherein the scale of error depends on the magnitude of the statistical query, thereby affording a small error when $f(D)$ is small.

\paragraph{Our Contributions.}
We consider the problem of answering statistical queries with {\em relative} error. Namely, we say that the estimates $\hat{\be} := (\he_f)_{f\in \cF}$ are \emph{$(\zeta,\alpha, \cF)$-accurate} with respect to dataset $D$ if for all $f \in \cF$ it holds that $(1 - \zeta) \cdot f(D) - \alpha \leq \he_f \leq (1 + \zeta) \cdot f(D) + \alpha$.
A (randomized) mechanism $\cA:\cX^* \mapsto\R_{\ge 0}^\cF$ is \emph{$(\zeta, \alpha, \beta, \cF)$-accurate} if $\cA(D)$ is $(\zeta, \alpha, \cF)$-accurate with respect to $D$ with probability at least $1 - \beta$.
As before, a stronger approach is to generate a synthetic dataset $\widehat{D}$ such that the derived estimates $\he_f := f(\widehat{D})$ are $(\zeta, \alpha, \beta, \cF)$-accurate with respect to $D$.

In \Cref{sec:upper-bounds}, we introduce a new framework 
of $\mathsf{PREM}$ ({\bf P}rivate {\bf R}elative {\bf E}rror {\bf M}ultiplicative weight update) that, under $(\eps,\delta)$-DP, achieves $\alpha = \tO\Big(\frac{1}{\zeta\eps} \prn{\log n \log \frac1\delta}^{\frac32} \sqrt{\log |\cX|}\log|\cF| \Big)$. Notice that this error is only polylogarithmic in $n, \frac{1}{\delta}, |\cX|, |\cF|$, which is in stark contrast to the aforementioned additive-only results, which \emph{must} be polynomial in one of the parameters.

Under $\eps$-DP, our PREM framework achieves  
$\alpha = \tO\big(\sqrt{\frac{n}{\eps \zeta^2}\log |{\cF}|\log|{\cX}|}+\frac{1}{\eps}\big)$. While this has polynomial dependency of $\sqrt{n}$ on the dataset size, it is still an improvement over the $n^{2/3}$ dependency of the best known mechanisms for additive error~\citep{BlumLR13,HardtLM12}.

In \Cref{sec:lower-bounds}, we prove lower bounds on the relative accuracy of any mechanism that releases estimates for statistical queries (not just synthetic data generation) under both approximate- and pure-DP. For $(\eps, \delta)$-DP, we show that for sufficiently large family of queries $\cF$, our algorithm is nearly optimal (up to polylog factors on $n,\log|\cF|,\log|\cX|,\frac{1}{\zeta},\frac{1}{\eps},\frac{1}{\delta}$). For $\eps$-DP, our lower bound is not as sharp compared to our upper bounds, which is a gap that is present even in the purely-additive case \citep{DPorg-open-problem-optimal-query-release}. We suspect that in order to tighten the pure-DP gaps for relative accuracy it may be necessary to first resolve the gaps on the purely-additive setting.

Finally, in \Cref{app:real-v-bin} we extend our relative-error upper bounds to hold for queries taking values in the interval $[0, 1]$. This approach is based on a general reduction that splits each real-valued statistical query into Boolean-valued threshold queries with exponentially-decreasing thresholds. This reduction induces a minimal overhead in the additive error bound. In fact, the additive error remains asymptotically the same as long as $\frac{\log n}{\zeta} \leq |\cF|^{O(1)}$. 

\paragraph{Related Work.}
Query answering and synthetic data generation are central problems in private data analysis, and there is a vast literature studying them.

The first concerns on answering counting queries over sensitive databases were the driving force behind the notion of DP~\citep{Dinur:2003,Dwork:2006}. Ever since, this problem has been a central focus in this area (see, e.g.,~\cite{Hay:2010,Li:2010matrix_mech,Hardt:2010,Gupta:2011,Nikolov:2013,Dwork:2015}). It was later observed that generating private synthetic data offers the possibility for data analysts to access datasets in `raw' form, and where any posterior analysis is privacy protected, by postprocessing properties of DP. A meaningful way to assess the quality of synthetic data is by its worst-case additive error over a set of prescribed (or adaptively generated) queries \citep{Barak:2007,BlumLR13,HardtR10,HardtLM12,GuptaRU12,Hsu:2013,Gaboardi:2014}. Both in the query answering and synthetic data settings, it is known that any DP algorithm must incur error that is polynomial in at least one of $n,|\cF|,|\cX|$ \citep{HardtPhDThesis:2011,Bun:2014}. 

Relative accuracy guarantees have become a recent focus in DP, particularly for analytics settings, where trends can be better traced by substantial changes, naturally expressed in relative terms \citep{Cormode:2012,Qardaji:2013,Zhang:2016PrivTree,EpastoMMMVZ23,Ghazi:2023}. 
However, existing works have focused on specific problems, and to our knowledge no general framework has been established in this context. We point out that among these specific settings of interest, the case of spectral and cut approximations on graphs has been studied under relative approximation in multiple works~\cite[e.g.,][]{Blocki:2012,Arora:2019}. Besides the setting of statistical queries, relative error is also the standard notion in approximation algorithms \citep[see, e.g.,][]{WS11}, which have been studied with DP~\citep{GuptaLMRT10}.
\section{Preliminaries}\label{sec:prelims}

We follow the common convention of representing a dataset $D = \{x_1, \ldots, x_n\}$, whose elements are from the domain $\cX$, by the corresponding histogram vector $\bh^D \in \Z_{\ge 0}^{\cX}$, 
where $\bh^D_x$ is $|\{ i \in [n] : x_i = x \}|$, as well as interpreting $f: \cX \to \{0, 1\}$ as a vector in $\{0,1\}^{\cX}$. 
It is immediate to see that
$
f(D) = \ang{\bh^D, f} = \sum_{x \in \cX} h^D_x \cdot f(x)
$.
For simplicity, we use $\bh$ to denote the dataset itself and skip the superscript $D$, and we use $f(\bh)$ to denote $\ang{\bh, f}$. For any $S \subseteq \cX$, we use $\ind_S : \cX \to \{0, 1\}$ to denote the function $\ind_S(x) = \mathds{1}\{x \in S\}$. And so, $\ind_S(\bh) = \sum_{x \in S} h_x$. And we use $\bh|_S$ to denote the histogram $\bh' \in \Z_{\ge 0}^{\cX}$ with $h_x' = h_x \cdot \ind_S(x)$.

We consider mechanisms $\cB:\Z_{\ge 0}^{\cX}\mapsto \R_{\ge 0}^{\cX}$ that generate synthetic histograms $\hat{\bh} = \cB(\bh)$, from which the estimates $(\hat{e}_f := f(\hat{\bh}))_{f \in \cF}$ can be derived. We refer to such mechanisms as 
{\em synthetic data generators}.
Note that, we allow the synthetic histogram to be real-valued (instead of integer-valued). 
 In what follows, we will not distinguish between a synthetic histogram and a synthetic dataset, for the following reason.
 If a (nonnegative real-valued) histogram is relatively accurate with respect to a query family, sampling a dataset i.i.d.~from the probability distribution induced by the histogram will also be relatively accurate, with slightly worse parameters. We omit these details, but they are implicit in the proof of \Cref{prop:existence_approx_histogram}.

\paragraph{Differential Privacy.}
Two data histograms $\bh, \bh' \in \Z_{\ge 0}^{\cX}$ are said to be {\em adjacent} if
$\|\bh\|_1 = \|\bh'\|_1$ and $\|\bh - \bh'\|_1 = 2$; 
this is equivalent to replacing some record in the dataset corresponding to $\bh$ by a different record in the dataset corresponding to $\bh'$. We view randomized mechanisms $\cM$ as mapping datasets $\bh$ to a random variable $\cM(\bh)$ over some output space $\cO$.

\begin{definition}[{\boldmath $(\eps, \delta)$-DP}]
    A randomized mechanism $\cM$ with output space $\cO$ satisfies \emph{$(\eps, \delta)$-DP} (referred to as \emph{approximate-DP}) if, for all (measurable) events $E\subseteq \cO$, and for all adjacent $\bh, \bh'$ it holds that
    $
    \Pr[\cM(\bh) \in E] \le e^\eps \cdot \Pr[\cM(\bh') \in E] + \delta
    $.
    The special case of $\delta = 0$ is denoted as $\eps$-DP (referred to as \emph{pure-DP}).
\end{definition}

\begin{algorithm}[t]
\small
\caption{\textsc{$\OutsideIntervalMonitor^{\mathrm{apx}}_{\bh, a, \cY}$ (Approximate-DP version)}}
\label{alg:outside-thresholds}

\textbf{Initialization:}
{\footnotesize $\bullet$ } $\bh \in \Z_{\ge 0}^{\cX}$ : private histogram \qquad \qquad 
{\footnotesize $\bullet$ } $a > 0$ : noise parameter\\
\phantom{\textbf{Initialization:}\,}
{\footnotesize $\bullet$ } $\cS_{\acti} \gets \cY$ : initial active set (for $\cY \subseteq \cX$)\\[-2mm]

\textbf{On input} $(f: \cX \to \{0, 1\}, \tau_\ell, \tau_u \in \R_{\ge 0})$ {\bf :}\\[-2mm]

$\hf \gets f(\bh|_{\cS_{\acti}}) + \Lap(1/a)$\;

\If{{\em $\tau_\ell < \hf < \tau_u$}} {
    \Return{\Between}\;
}
\Else{
    ${\cal S}_{\acti} \gets {\cal S}_{\acti} \smallsetminus f^{-1}(1)$\;\algcomment{Note: $\Sactive$ is a persistent state across inputs.}\\
    \If{$\hf \ge \tau_u$} {
        \Return{\Above}
    } \Else {
        \Return{\Below}
    }
}
\end{algorithm}

We use the compositional properties of DP and the well-studied Laplace mechanism; we include the details in \Cref{sec:basic_DP_properties} for completeness.

\paragraph{$\OutsideIntervalMonitor$.}
We use the $\OutsideIntervalMonitor$, an iterative mechanism that determines if a given statistical query evaluates to a value within a specified interval or falls above or below it. The key property is that the privacy cost of the entire mechanism on a sequence of queries only degrades with the number of times the value is {\em not} in the interval. We use this technique in two flavors: one that satisfies approximate-DP (\Cref{alg:outside-thresholds}) and other that satisfies pure-DP (\Cref{alg:outside-thresholds-pure-dp}). The analysis of the pure-DP version follows from the standard analysis of the so-called {\em sparse vector technique (SVT)} (\Cref{sec:OIM_pure_DP}). The approximate-DP version is analyzed using the target charging technique~\citep{Cohen:2023} and individual privacy accounting (\Cref{app:TCT}).

\begin{proposition}[{\boldmath $\OutsideIntervalMonitor$ (Approx-DP) guarantees}]\label{prop:oim-approx-dp} 
For any $0< a\leq 1$, $\delta>0$ and integer $R > 0$, $\OutsideIntervalMonitor^{\mathrm{apx}}_{\bh, a, \cY}$ (\Cref{alg:outside-thresholds}), after $R$ rounds of queries, satisfies $(\eps, \delta)$-DP for $\eps = O\left(a \log \frac{1}{\delta}\right)$.  Let $\beta > 0$, $C := \frac1a \log \frac{R}{\beta}$, and $\Sactive$ be the state of the set before the query is made.  Then, with probability at least $1-\beta$, on any query $(f, \tau_\ell, \tau_u),$
\begin{itemize}[topsep=3pt,itemsep=-3pt]
    \item If $\Between$ is returned, then $\tau_\ell - C \le f(\bh|_{\cS_{\acti}}) \le \tau_u + C$.
    \item If $\Above$ is returned, then $f(\bh|_{\cS_{\acti}}) \ge \tau_u - C$.
    \item If $\Below$ is returned, then $f(\bh|_{\cS_{\acti}}) \le \tau_\ell + C$.
\end{itemize}
\end{proposition}

\section{Algorithms for Synthetic Histograms with Relative Accuracy}\label{sec:upper-bounds}

In this section, we introduce the $\PREM$ framework for generating synthetic histograms with relative error guarantees, starting with the case of approximate-DP.

\subsection{Approximate-DP Algorithm}

\label{sec:approx_DP_synthetic_data}

\begin{theorem} \label{thm:main}
For all $0< \eps \leq 1$, $\delta \in (0, 1)$, $\beta \in (0, 1)$, and $\zeta \in (0, \nicefrac12)$, 
there is an $(\eps, \delta)$-DP $(\zeta,\alpha,\beta,\cF)$-accurate synthetic data generator for any domain $\cX$ and query family $\cF$ with
\[
    \alpha = \tO\prn{\frac{1}{\zeta\eps} \cdot \Big(\log n \cdot  \log\frac1\delta \Big)^{\frac32} \cdot \log^{\frac12} |\cX| \cdot \log \prn{\textstyle \frac{|\cF|}{\beta}}
    }.
\]
\end{theorem}

\paragraph{Proof Intuition.} 
Let us start by recalling the PMWU framework  of~\citet{HardtR10,HardtLM12}, which uses a private version of the multiplicative weights update (MWU) rule and yields nearly tight guarantee in the additive-only setting. At a high level, this algorithm starts with a synthetic dataset $\hbh$ and iteratively updates it. In each iteration, PMWU identifies a query $f \in \cF$ with large error. The synthetic dataset $\hbh$ is then updated using an MWU rule.

A key parameter that governs the number of required iterations is the \emph{margin}, which can be defined as the normalized error of $f$ between the synthetic dataset $\hbh$ and the true dataset $\bh$  (i.e., $|f(\hbh) - f(\bh)| / n$). When the margin is $\gamma$, the number of required updates is $\tO(1/\gamma^2)$. 
As discussed in \Cref{tab:prior-work-our-results}, all previous works~\citep{HardtR10,HardtLM12,GuptaRU12} incur error that is polynomial in $n$ 
because the number of updates required in their  algorithm is polynomial in $n$. 
This is because the guaranteed margin of a large additive-gap query is only 
$\alpha / n$, which leads to a bound of 
$\tOmega((n/\alpha)^2)$ iterations in PMWU. 

At a high level, our main observation is that we can make this margin $\zeta$ instead of $\alpha/n$, allowing us to (drastically) 
reduce the number of iterations to $\tO(1/\zeta^2)$. 
However, achieving such a margin is quite challenging: previous frameworks~\citep{HardtR10,HardtLM12,GuptaRU12} use either the 
exponential mechanism or \textsc{SVT} to select a \emph{single} query from $\cF$ for which the current distribution performs poorly and use it
as the example for update. Unfortunately, this does not suffice: if all counting queries are sparse (e.g.,~counting a single domain element), 
then it is possible that the margin of any query in $\cF$ here is only $O(\alpha/n)$.
Our key technique to overcome this is the observation that when queries are sparse, we can run $\OutsideIntervalMonitor$, which using the target charging can select \emph{multiple} queries that---after appropriate corrections from previous queries---differ by more than $1 \pm \zeta$ from the target value. Roughly speaking, this allows us to combine all these queries together to construct a bad ``example'' with margin $\Omega(\zeta)$.

\subsubsection{Finding a Bad Margin Example with $\OutsideIntervalMonitor$}

As discussed above, a key ingredient of our approach is identifying a bad ``example'' i.e.~a large support set where the candidate histogram fails to multiplicatively approximate the true counts; if such example does not exist, we certify a uniform relative approximation for all queries over a large set. We present this subroutine, \Cref{alg:find-example}, together with its privacy and accuracy guarantees. This algorithm maintains two violating sets, $\cS^+$, $\cS^-$: these sets aggregate the supports of queries that violate the respective upper and lower bounds imposed by the desired relative approximation. 

\begin{algorithm}[t]
\small
\caption{$\FindMarginExample_{\cF,\eps,\delta,\beta,\zeta}$}
    \label{alg:find-example}
        \textbf{Parameters:}
            {\footnotesize $\bullet$ } privacy parameters $ \eps >0$ and $0 < \delta < 1$; 
            \,
            {\footnotesize $\bullet$ } confidence parameter $0 < \beta < 1$;\\
        \phantom{{\bf Parameters}:}
            {\footnotesize $\bullet$ } approximation factor $0 < \zeta < 1 / 2$;
            \quad \quad \quad
            {\footnotesize $\bullet$ } set of counting queries $\cF\subseteq \{0, 1\}^{\cX}$;\\[-2mm]
        
        \textbf{Input:}
        {\footnotesize $\bullet$ } private input histogram $\bh^{\ast} \in \Z_{\ge 0}^{\cX}$ with $\|\bh^{\ast}\|_1=n$;\\
        \phantom{{\bf Input:}}
        {\footnotesize $\bullet$ } currently estimated histogram $\hbh\in \R_{\geq 0}^{\cX}$;\\
        \phantom{{\bf Input:}}
        {\footnotesize $\bullet$ } active set $\cY\subseteq \cX$;\\[-2mm]
        
        $\Splus, \Sminus \gets \emptyset$ and 
        $\Sactive \gets \cY$\\
        $\OIM \gets \OutsideIntervalMonitor^{\mathrm{apx}}_{\bh^{\ast},a,\cY}$ for $a = \eps / O(\log(1/\delta))$\\[1mm]
        \mbox{}\algcomment{initialized with $a$ such that $\OIM$ satisfies $(\eps, \delta)$-DP (via \Cref{prop:oim-approx-dp})}\\
        $\alpha_0\gets \frac{4(1+\zeta)}{a}\ln\big(\frac{|\cF|}{\beta}\big)$\algcomment{equals $2(1+\zeta)C$ for $C$ given in \Cref{prop:oim-approx-dp} for $R = |\cF|^2$.}\\
        $\Factive\gets \cF$\; \hfill\algcomment{maintains a set of queries not already handled}
        
        \Repeat{\textsc{Accurate}}{
            \textsc{Accurate} $\gets$ \textsc{True}
            
            \For{$f\in\Factive$}{
                $\tau_u \gets \frac{1}{(1-\zeta)}\big( f(\hbh|_{\Sactive})+\frac{\alpha_0}{2}\big)$ and $\tau_\ell \gets \frac{1}{(1+\zeta)}\big(f(\hbh|_{\Sactive})-\frac{\alpha_0}{2}\big)$

                $s \gets \OIM(f, \tau_u, \tau_\ell)$

                \If{$s \ne \Between$}{
                    \If{$s = \textsc{Above}$}{
                        $\Splus\gets\Splus\cup(\Sactive\cap f^{-1}(1))$\;
                    }
                    \Else {
                        $\Sminus\gets\Sminus\cup(\Sactive\cap f^{-1}(1))$\;
                    }
                    
                    $\Sactive\gets \Sactive \smallsetminus f^{-1}(1)$\; \algcomment{Identical to $\Sactive$ maintained in state of $\OIM$.}\\
                    $\Factive\gets \Factive\smallsetminus \{f\}$\;\\
                    \textsc{Accurate} $\gets$  \textsc{False}\;
                }
            }
        }
        \Return{$(\theta,\cS)\gets \begin{cases}
            (\plus,\Splus) & \text{\em if } \ind_{\Splus}(\hbh) \ge \max\{ \ind_{\Sminus}(\hbh), \ind_{\Sactive}(\hbh) \} \\
            (\minus,\Sminus) & \text{\em if } \ind_{\Sminus}(\hbh) \ge \max\{ \ind_{\Splus}(\hbh), \ind_{\Sactive}(\hbh) \}  \\
            (\equal, \Sactive) & \text{\em if } \ind_{\Sactive}(\hbh) \ge \max\{ \ind_{\Sminus}(\hbh), \ind_{\Splus}(\hbh) \} 
            \end{cases}$
        }
\end{algorithm}

\begin{lemma} \label{lem:bad-margin}
    For all $\cF$, $0<\eps \leq 1$, 
    $\delta, \beta, \zeta \in (0, 1)$, $\FindMarginExample_{\cF,\eps,\delta,\beta,\zeta}$ (\Cref{alg:find-example}) satisfies $(\eps, \delta)$-DP. 
    Moreover, given active set $\cY\subseteq\cX$, private histogram $\bh^{\ast}\in\N_+^{\cX}$, 
    and an estimate histogram $\hbh\in\R_+^{\cX}$, the output $(\theta, \cS)$ satisfies:
    \begin{itemize}
    \item $\ind_\cS(\hbh) \geq \frac{1}{3} \ind_{\cY}(\hbh)$.
    \item With probability $1-\beta$ (see the pseudocode for the value of $\alpha_0$):
        \begin{itemize}
            \item If $\theta = \equal$, for all $f\in \cF$, $(1 - \zeta) \cdot f(\bh^{\ast}|_{\cS}) - \alpha_0 \leq f(\hbh|_{\cS})  \leq (1 + \zeta) \cdot f(\bh^{\ast}|_{\cS}) + \alpha_0$.
            \item If $\theta = \plus$, $\ind_{\cS}(\bh^{\ast})  \geq (1 + \zeta) \cdot \ind_{\cS}(\hbh)$.
            \item If $\theta = \minus$, $\ind_{\cS}(\bh^{\ast}) \leq (1 - \zeta/2) \cdot \ind_{\cS}(\hbh)$.
        \end{itemize}
    \end{itemize}
\end{lemma}

\begin{proof}
The privacy guarantee follows immediately from \Cref{prop:oim-approx-dp} (for suitable $a = \eps / O(\log \frac1\delta)$), since the private histogram $\bh^\ast$ is only accessed through queries to $\OutsideIntervalMonitor$.

Since $\Splus$, $\Sminus$, $\Sactive$ form a partition of $\cY$, we have $\max\{ \ind_{\Splus}(\hbh), \ind_{\Sminus}(\hbh), \ind_{\Sactive}(\hbh) \}\geq \|\hbh\|_1/3$. And hence $\ind_{\cS}(\hbh) \geq \frac{1}{3}  \ind_\cY(\hbh)$.

Next, we proceed to the second part. Note that there can be at most $|\cF|$ iterations in the \textbf{repeat}--\textbf{until} loop, since $|\Factive|$ decreases every iteration if the loop does not terminate. Therefore, the number of queries made to $\OIM$ is at most $|\cF|^2$. Thus, from \Cref{prop:oim-approx-dp}, with probability at least $1 - \beta$, all the responses of $\OIM$ are accurate to within 
$C = \frac2a \log \frac{|\cF|}{\beta} = \frac{\alpha_0}{2(1+\zeta)}$. 
Conditioned on this event, we obtain the accuracy guarantee for each of the three cases of $\theta$ as follows:
    \begin{description}[leftmargin=5pt]
        \item[{\boldmath Case $\theta=\textsc{Approx}$}:] Let $\Factive$ refer to the state of this set at the end of the algorithm. For $f\notin\Factive$ it holds that $f(x) = 0$ for all $x \in \Sactive$, and hence $f(\hbh|_{\Sactive}) = f(\bh^{\ast}|_{\Sactive}) = 0$.

        In the last loop of \textbf{repeat}--\textbf{until}, we have that $\OIM(f, \tau_u, \tau_\ell)$ returns a response of $\Between$ for all $f \in \Factive$. Hence conditioned on the accuracy guarantee of $\OIM$, we have that
            \begin{eqnarray*}
                \frac{1}{1+\zeta}\left( f(\hbh|_{\Sactive}) - \frac{\alpha_0}{2} \right) - C\leq & 
                f(\bh^{\ast}|_{\Sactive})
                & \leq \frac{1}{1-\zeta}\left( f(\hbh|_{\Sactive}) + \frac{\alpha_0}{2} \right) + C \\
                \Longrightarrow
                (1-\zeta)f(\bh^{\ast}|_{\Sactive}) -\alpha_0 \leq 
                & f(\hbh|_{\Sactive}) &
                \leq (1+\zeta)f(\bh^{\ast}|_{\Sactive}) + \alpha_0.
            \end{eqnarray*}
            We conclude that $\hbh|_{\Sactive}$ is $(\zeta,\alpha_0,\Factive)$-accurate, and in conjunction with the perfect accuracy over $\cF\setminus\Factive$ concludes the proof of the claim.
        \item[{\boldmath Case $\theta=\plus$}:] Conditioned on the accuracy guarantee of $\OIM$, at every query $(f,\tau_\ell,\tau_u)$ where $\OIM$'s output is $\Above$, we have that
        \begin{align*}
        \textstyle
        \ind_{\Sactive\cap f^{-1}(1)}(\bh^{\ast})
            & \textstyle = f(\bh^{\ast}|_{\Sactive})
            ~\geq~ \frac{1}{(1-\zeta)}\big( f(\hbh|_{\Sactive}) + \frac{\alpha_0}{2} \big) - C\\
            &\textstyle\geq \frac{1}{(1-\zeta)} \cdot   \ind_{\Sactive\cap f^{-1}(1)}(\hbh) 
            ~\geq~ (1+\zeta) \cdot \ind_{\Sactive\cap f^{-1}(1)}(\hbh),
        \end{align*}
        where $\Sactive$ is the state before the query is performed.
            Finally, since the subsets $\Sactive\cap f^{-1}(1)$
        obtained are disjoint across iterations, and their union is $\Splus$, adding up these inequalities we get
            $\ind_{\Splus}(\bh^{\ast}) 
            \geq (1+\zeta) \cdot \ind_{\Splus}(\hbh)$.
        \item[{\boldmath Case $\theta=\minus$}:] Conditioned on the accuracy guarantee of $\OIM$, at every query $(f, \tau_\ell, \tau_u)$ where $\OIM$'s output is $\Below$, we have 
        (since $\alpha_0 = 2(1+\zeta) C$), that
        \begin{align*}
            \textstyle \ind_{\Sactive\cap f^{-1}(1)}(\bh^{\ast})
            &\textstyle= f(\bh^{\ast}|_{\Sactive})
            \leq \frac{1}{(1+\zeta)}\big( f(\hbh|_{\Sactive}) - \frac{\alpha_0}{2} \big) + C\\
            &\textstyle\leq \frac{1}{(1+\zeta)} \cdot  \ind_{\Sactive\cap f^{-1}(1)}(\hbh) 
            \leq (1-\frac\zeta2) \cdot \ind_{\Sactive\cap f^{-1}(1)}(\hbh),
        \end{align*}
        where $\Sactive$ is the state before the query is performed.
        Finally, since the subsets $\Sactive\cap f^{-1}(1)$
        obtained are disjoint across iterations, and their union is $\Sminus$, adding up these inequalities we get
        $
            \ind_{\Sminus}(\bh^{\ast})  
            \leq (1-\frac\zeta2) \cdot \ind_{\Sminus}( \hbh)
        $.
    \end{description}
    \vspace{-6mm}
\end{proof}

\subsubsection{$\PREM$ and Proof of \Cref{thm:main}}

We now present $\PREM$ (\Cref{alg:dp-mwu}) for $\delta>0$ (the case of $\delta=0$ is considered in \Cref{sec:pure_DP_UB}), that underlies \Cref{thm:main}. We sketch the main idea at an intuitive level.
This algorithm operates in multiple rounds, which work as follows with high probability.

At the start of round $i$, an ``active set'' $\cX^i$ is maintained such that we have an estimate $\hbh$ supported on ${\cX \smallsetminus \cX^i}$ satisfying $\hbh \approx_{\cF,\zeta,\alpha_*} \bh^*|_{\cX \smallsetminus \cX^i}$ (denoting that $f(\hbh) \in (1 \pm \zeta) f(\bh^{\ast}|_{\cX \smallsetminus \cX^i}) \pm \alpha_*$ for all queries $f \in \cF$).
At the end of round $i$, we identify a new set $\cS^i \subseteq \cX^i$ such that $\bh^{\ast}|_{\cS^i}$ contains a constant fraction of the active set mass of $\bh^{\ast}|_{\cX^i}$, and an estimate $\hbh^{i}$ supported only on $\cS^i$ such that $\hbh^{i} \approx_{\cF, \zeta, \alpha_0} \bh^{\ast}|_{\cS^i}$. By setting $\hbh \gets \hbh + \hbh^i$ and $\cX^{i+1} \gets \cX^i \smallsetminus \cS^i$, the invariant is maintained (for $\alpha_* \gets \alpha_* + \alpha_0$), and the active set mass shrinks by a constant factor. Thus, in at most $I = O(\log n)$ rounds, most of the mass of $\bh^*$ is accounted for.

To identify the subset $\cS^i$ and the estimate $\hbh^i$ at every round, we start with $\hbh^i$ being uniform over $\cX^i$ and use $\FindMarginExample$ (\Cref{alg:find-example}) iteratively to either update $\hbh^i$ using a MWU or certify its accuracy on a large support $\cS^i$. Using a potential function argument, we show that this terminates in a small number of rounds with high probability.
\begin{algorithm}[t]
\small
    \caption{$\PREM_{\cF,\eps,\delta,\beta,\zeta}$ : \textsc{Private Relative Error MWU}}
    \label{alg:dp-mwu}
        \textbf{Parameters:}
            {\footnotesize $\bullet$ } privacy parameters $\eps > 0$ and $0 \leq \delta < 1$;
            \ \ {\footnotesize $\bullet$ } confidence parameter $0 < \beta < 1$;\\
        \phantom{\textbf{Parameters:}}
            {\footnotesize $\bullet$ } approximation factor $0 < \zeta < 1 / 2$;
            \quad \quad \quad
            {\footnotesize $\bullet$ } query family $\cF\subseteq \{0, 1\}^{\cX}$;

        \textbf{Input:} input histogram $\bh^{\ast} \in \N^{\cX}$ with $\|\bh^{\ast}\|_1=n$.
        
        $I \gets \lceil \frac{\log n}{\ln 6/5}\rceil$\\
        $T \gets \lceil\frac{128\log |\cX|}{\zeta^2}\rceil$

        $\eta\gets \zeta/4$; $\beta'\gets \beta/(2IT)$; $\cX^1 \gets \cX$
        
        \If{$\delta=0$}{
            $\delta'\gets 0$;\,\, $\eps'\gets\eps/(2IT)$\;\\
            $a\gets \eps/I$;\,\, $\alpha\gets \tilde O\Big( \max\big\{\sqrt{\frac{n \log^3 n}{\zeta^2\eps}\log|\cX|\log\big(\frac{|\cF|}{\beta}\big)},\frac{\log n}{\eps} \big\} \Big)$\;
        }\Else{
            $\delta'\gets\delta/(4IT)$\;\\
            Let $\eps'>0$ be the unique solution to $\frac{\eps}{2} =\eps' O(\log(1/\delta'))\big( \sqrt{2IT\ln(1/\delta')}+IT\big(\frac{e^{\eps'}-1}{e^{\eps'}+1}\big)\big)$\\
            \mbox{}\algcomment{$IT$-fold composition of $\OutsideIntervalMonitor$ must satisfy $(\eps/2,\delta/2)$-DP.}\\
            $a\gets \eps/(4\sqrt{2I\log(I/\delta)})$; 
            $\alpha \gets 
            \frac{200 I}{\eps'}\log\big(\frac{4|\cF|}{\beta'}\big)$\;
        }

        $\hbh \gets {\bm 0} \in \Z_{\ge 0}^{\cX}$\algcomment{Current estimate of the histogram}

        \For{$i = 1$ to $I$}{
            $\tilde n_i \gets \|\bh^{\ast}|_{\cX^i}\|_1+\mbox{Lap}\big(1/a\big)$

            \If{$\tilde n_i\leq \alpha/4$}{
                \textsc{Break}
            }
            $\hbh^{i} \gets \frac{\tilde n_i}{|\cX^i|}\ind_{\cX^i}$ \algcomment{Initialize as uniform over active set, with roughly the right total mass}\\
            \For{$t = 1,\ldots,T+1$}{
        	\If{$t=T+1$}{
        		  \Return{\textsc{Failure}}
        	}
        	$(\theta^i, \cS^i) \gets \textsc{FindMarginExample}_{\cF,\eps',\delta',\beta',\zeta}(\bh^{\ast}, \hbh^{i}, \cX^i)$\\
        	\If{$\theta^i = \equal$}{
                        $\hbh \gets \hbh + \hbh^{i}|_{\cS^i}$\\
            		$\cX^{i+1} \gets \cX^i \setminus {\cS^i}$\\
            		\textsc{Break}
        	}
                \Else{
                    $\hbh^{i} \gets \tilde n_i \cdot \frac{\hbh^{i}\odot \exp\big\{\theta^i\eta\ind_{\cS^i}\big\}}{\|\hbh^{i}\odot \exp\big\{\theta^i\eta\ind_{\cS^i}\big\}\|_1} $\algcomment{where $a \odot b$ denotes pointwise product of $a, b \in \R_{\ge 0}^{\cX}$.}
                }
            }
        }

        \Return{$\hbh$}
\end{algorithm}

Before proving the theorem, we establish a non-failure probability guarantee for the algorithm.

\begin{claim} \label{claim:failure_PREM}
\Cref{alg:dp-mwu} returns \textsc{Failure} with probability at most $2\beta'IT$.
\end{claim}

\begin{proof}
First, notice that as $i$ plays no role in this analysis, we will omit its dependence. In this context, we work on the restricted sample space $\cX\gets\cX^i$, together with private data $\bh^{\ast}\gets \bh^{\ast}|_{\cX^i}$, sample size $n\gets n_i:=\|\bh^{\ast}|_{\cX^i}\|_1$ and $\tilde n\gets \tilde n_i=\|\hbh^{i,t}\|_1$. Note that by the properties of Laplace noise, with probability $1-I\beta'$, 
\begin{equation} \label{eqn:Laplace_concentration_MWU}
\textstyle |n_i-\tilde n_i|\leq \frac{4\sqrt{2I\ln(I/\delta)}}{\eps}\ln\big(\frac{1}{\beta'}\big) \quad(\forall i=1,\ldots,I). 
\end{equation}
This concentration bound implies that if $\tilde n_i> \alpha/4$ then $n_i> \alpha/6$, by the definition of $\alpha$.

Next we proceed with a large-margin analysis of the MWU method. 
In particular, \Cref{alg:find-example} provides the following guarantee (\Cref{lem:bad-margin}): denoting by $(\theta_t,\cS_t)$ its output and $\hbh^t$ the histogram $\hbh^i$, at their $t$-th inner iteration, 
if $\theta_t\in\{\plus,\minus\}$ then with probability $1-\beta'IT$, for all $t,i$ 
\begin{align} 
\theta_t \left( \ind_{\cS_t}(\bh^{\ast}) - \ind_{\cS_t}(\hbh^t) \right) &\geq \frac{\zeta}{2} \ind_{\cS_t}(\hbh^t)  \label{eqn:large_margin_convergence} \\
\ind_{\cS_t}(\hbh) &\geq \frac13 \ind_{\cX}(\hbh). \label{eqn:large_weight_convergence}
\end{align}
We now proceed with a potential function analysis for upper bounding the failure probability~\citep[see, e.g., Chapter 7 in][]{Mohri:2018}. Consider the
potential function $\Phi_t=\mbox{KL}\big(\frac{\bh^{\ast}}{\|\bh^{\ast}\|_1}\big|\big|\frac{\hbh^t}{\|\hbh^t\|_1}\big)=\sum_{x\in \cX} \frac{\bh^{\ast}(x)}{n}\ln\Big( \frac{\bh^{\ast}(x)/n}{\hbh^t/\tilde n} \Big)$. 
Note that $\Phi_1 \le \ln |\cX|$ since $\hbh^t / \|\hbh^t\|_1$ is the uniform distribution and $\Phi_{T+1} \ge 0$. Hence. $\Phi_{T+1}-\Phi_1 \geq - \ln|\cX|.$ On the other hand, under the events established above:
\begin{align*}
\Phi_{t+1}-\Phi_t 
&= \sum_{x\in \cX} \frac{\bh^{\ast}(x)}{n}\ln\Big( e^{-\theta_t\eta \ind_{\cS_t}(x)} \sum_{y\in \cX}\frac{\hbh^t(y)}{\tilde n}e^{\theta_t\eta \ind_{\cS_t}(y)} \Big) \\
&=-\sum_{x\in \cX}\frac{\bh^{\ast}(x)}{n}\cdot\theta_t\eta\ind_{\cS_t}(x)+\ln\Big( \sum_{y\in \cX} \frac{\hbh^t(y)}{\tilde n}\exp\{\theta_t\eta\ind_{\cS_t}(y)\} \Big).
\end{align*}
Let now $P_t$ be the probability over $\cX$ where $P_t(x)=\hbh^t(x)/\tilde n$, and let $\mu_t=\mathbb{E}_{y\sim P_t}[\theta_t\eta \cdot \ind_{\cS_t}(y)]=\theta_t\eta \cdot \ind_{\cS_t}(\frac{\hbh^t}{\tilde n})$. Hoeffding's bound in \citet[Lemma B.7]{ShalevShwartz:2014} implies 
\[ \ln\Big( \sum_{y\in \cX} \frac{\hbh^t(y)}{\tilde n}\exp\{\theta_t\eta\ind_{\cS_t}(y)-\mu_t\} \Big) = \ln\Big( \mathbb{E}_{y\sim P_t} \big[\exp\{\theta_t\eta   \ind_{\cS_t}(y)-\mu_t\}\big]\Big) \leq \frac{\eta^2}{8}. \]
Re-arranging the potential drop to incorporate this term, we conclude that
\begin{align*}
\Phi_{t+1}-\Phi_t 
&\leq-\theta_t\eta \cdot \left( \ind_{\cS_t}\left(\frac{\bh^{\ast}}{n}\right) - \ind_{\cS_t}\left(\frac{\hbh^t}{\tilde n}\right) \right)
+\frac{\eta^2}{8}\\
&= -\frac{\theta_t\eta}{n} \cdot ( \ind_{\cS_t}(\bh^{\ast}) - \ind_{\cS_t}(\hbh^t) ) - \theta_t\eta \cdot \left(\frac1n-\frac{1}{\tilde n}\right) \cdot \ind_{\cS_t}(\hbh^t) + \frac{\eta^2}{8}.
\end{align*}
Now to bound the resulting expression, note that the first summand can be bounded using \eqref{eqn:large_margin_convergence}, 
whereas for the second one we can use \eqref{eqn:Laplace_concentration_MWU} to conclude that
\[\textstyle - \theta_t\eta\big(\frac1n-\frac{1}{\tilde n}\big) \cdot \ind_{\cS_t}(\hbh^t) \leq \frac{\eta}{n\tilde n} \frac{4\sqrt{2I\ln(I/\delta)}}{\eps}\log \frac{1}{\beta'} \cdot \ind_{\cS_t}(\hbh^t). \]
Now, noting that upon non-termination of the algorithm, $\tilde n > \frac{\alpha}{4} \geq \frac{64\sqrt{2I\ln(I/\delta)}}{\zeta\eps}\ln \frac{1}{\beta'}$, we have
\begin{align*} 
&\textstyle-\frac{\zeta\eta}{2n} \cdot \ind_{\cS_t}(\hbh^t) +\frac{\eta}{n\tilde n} \frac{4\sqrt{2I\ln(I/\delta)}}{\eps}\log \frac{1}{\beta'} \cdot \ind_{\cS_t}(\hbh^t) \textstyle = -\frac{\zeta\eta}{2n}\Big(1-\frac{8\sqrt{2I\ln(I/\delta)}}{\zeta\tilde n\eps}\log \frac{1}{\beta'} \Big) \cdot \ind_{\cS_t}(\hbh^t) \\
& \leq -\frac{3\zeta \eta}{8n} \cdot \ind_{\cS_t}(\hbh^t) \textstyle \leq -\frac{\zeta \eta}{8n} \cdot \ind_{\cX}(\hbh^t) \leq-\frac{\zeta \eta}{8n} \Big( n-\frac{4\sqrt{2I\ln(I/\delta)}}{\eps}\ln \frac{1}{\beta'} \Big) 
\leq  -\frac{\zeta \eta}{16}, 
\end{align*}
where in the third inequality we used again the bound \eqref{eqn:Laplace_concentration_MWU}, and in the last one the fact that upon non-termination and \eqref{eqn:Laplace_concentration_MWU}, $n\geq \frac{\alpha}{6}\geq \frac{16\sqrt{2I\ln(I/\delta)}}{\eps}\log(1/\beta')$. 

By the choice $\eta=\frac{\zeta}{4}$, we get $-\ln|\cX| \leq \Phi_{t+1}-\Phi_1\leq -\frac{(t+1)\zeta^2}{128}$. 
We conclude that under the established event (which happens with probability at least $1-2\beta'TI=1-\beta$) the algorithm must break its inner loop within $t\leq T$ steps. This concludes the proof of the claim. 
\end{proof}

\begin{proofof}{\Cref{thm:main}}
The privacy of \Cref{alg:dp-mwu} follows from the privacy of \Cref{alg:find-example} and the adaptive composition theorem (\Cref{prop:dp-composition}). 
First, the counts $\tilde n_i$ incur a cumulative privacy budget of $(\eps/2,\delta/2)$ by adaptive composition; and second, our definition of $(\eps',\delta')$ is such that the total privacy budget incurred by the composition of $IT$ applications of $\OutsideIntervalMonitor$ is  $(\eps/2,\delta/2)$. 

We proceed to the accuracy analysis of the algorithm. Throughout, we condition on event \eqref{eqn:Laplace_concentration_MWU}. 
We first observe that  the weight of the restricted histograms  decreases exponentially quickly, which gives a bound on the number of steps until the outer loop breaks. First, by definition of $\cX^i$:
\begin{align*} 
\|\bh^{\ast}|_{\cX^i}\|_1-\|\bh^{\ast}|_{\cX^{i+1}}\|_1 =\|\bh^{\ast}|_{\cS^i}\|_1
\geq \frac{\|\bh^{\ast}|_{\cX^i}\|_1}{6}-\frac53\alpha_0,
\end{align*}
where the last step follows from \Cref{lem:bad-margin} as follows:\footnote{In the first inequality that follows, we are implicitly assuming that the ``all-ones'' query, $f\equiv \mathds{1}\in\cF$. This is without loss of generality, as it only increases the query family by one element.}
\begin{align*}
\textstyle \|\bh^{\ast}|_{\cS^i}\|_1 
&\textstyle = \ind_{\cS^i}(\bh^{\ast})   
\geq \frac{1}{1+\zeta}\left( \ind_{\cS^i}(\hbh^{i}) -\alpha_0 \right)
\geq 
\frac{1}{1+\zeta}\left(\frac{ \|\hbh^{i}\|_1 }{3} -\alpha_0\right) 
 =  \frac{1}{1+\zeta}\left(\frac{ \tilde n_i }{3} -\alpha_0\right) \\
&\textstyle \geq  \frac{1}{1+\zeta}\left(\frac{ \|\bh^{\ast}|_{\cX^i}\|_1 }{3} -\frac{4\sqrt{2I\ln(I/\delta)}}{3\eps}\ln \frac{1}{\beta'}-\alpha_0\right) 
\geq \frac{\|\bh^{\ast}|_{\cX^i}\|_1}{6}-\frac{5}{3}\alpha_0.
\end{align*}
This implies that $n_i:=\|\bh^{\ast}|_{\cX^i}\|_1$ satisfies the recurrence $n_{i+1} \leq \frac56 n_i+\frac53\alpha_0$, thus
$n_I \leq \big(\frac56\big)^In+10\alpha_0.$ 
In particular, selecting $I=\frac{\ln n}{\ln 6/5}$, we have that $n_I\leq 11\alpha_0\leq \alpha/8$ (where the last inequality is by definition of $\alpha$). Note that under \eqref{eqn:Laplace_concentration_MWU}, $n_i\leq \alpha/8$ implies $\tilde n_i\leq \alpha/4$, hence the outer loop breaks.

Let $i^*\in[I]$ be the iteration where the outer loop breaks. We claim that the sum of the final histograms, $\sum_{i< i^{\ast}}\hbh^{i}|_{\cS^i}$, is $(\zeta,2\alpha/3,\cF)$-accurate w.r.t. $\bh^{\ast}|_{\cY}$, where $\cY=\bigcup_{i<i^{\ast}}\cS^i$.

By \Cref{lem:bad-margin},
if the algorithm never fails, we obtain at every $i<i^{\ast}$ a pair $\hbh^{i},\cS^i$ such that 
\[ (1-\zeta) \cdot f(\bh^{\ast}|_{\cS^i}) - \alpha_0 \leq   f(\hbh^{i}|_{\cS^i}) \leq (1+\zeta) \cdot  f(\bh^{\ast}|_{\cS^i}) +\alpha_0 \quad(\forall f\in \cF),\]
where we recall that $\alpha_0=\frac{4(1+\zeta)}{a}\ln \frac{|\cF|}{\beta'}$ from \Cref{alg:find-example}. Moreover, 
adding up these inequalities, and noting that $\cY=\bigcup_{i<i^{\ast}} \cS^i$ be the (disjoint) union of the supports of $(\hbh^{i})_{i=1,\ldots,I}$, we get
\[ (1-\zeta) \cdot f(\bh^{\ast}|_{\cY}) - I\alpha_0 \leq f(\hbh) \leq (1+\zeta) \cdot f(\bh^{\ast}|_\cY) +I\alpha_0 \quad(\forall f\in \cF).\]
Notice that $I\alpha_0\leq 2\alpha/3$, by definition of $\alpha$. This proves our claim.

Next, 
note that the zero histogram 
is $(0,\alpha/3,\cF)$-accurate w.r.t.~$\bh^{\ast}|_{\cX_{i^{\ast}}}$. This  follows from the event \eqref{eqn:Laplace_concentration_MWU} and $\tilde n_{i^{\ast}}\leq \alpha/4$, which implies $\ind_{\cX^{i^*}}(\bh^{\ast}) \leq \alpha/4+4\sqrt{2I\ln(I/\delta)}\ln(1/\beta')/\eps\leq \alpha/3$. This and the previous claim show that $\sum_{i<i^{\ast}} \hbh^i$ is $(\zeta,\alpha,\cF)$-accurate w.r.t.~$\bh^{\ast}$, concluding the proof.
\end{proofof}

\subsection{Pure-DP Algorithm}

\label{sec:pure_DP_UB}

We provide a pure-DP synthetic data generation algorithm with relative accuracy guarantees. Our algorithm is based on the $\PREM$ framework, where we modify the $\FindMarginExample$ subroutine by one that uses SVT with pure-DP composition (details in \Cref{sec:OIM_pure_DP}).

We point out that very similar (but slightly sharper) rates from the ones we derive in this section can be obtained by applying the exponential mechanism in a similar spirit to the additive-only mechanism of~\cite{BlumLR13}. However, the running time of that approach is quasi-polynomial (in $n, |\cX|, |\cF|$) whereas our approach results in a polynomial-time algorithm (details in \Cref{app:exponential_mechanism}).

As mentioned earlier, the first building block is an adaptation of $\FindMarginExample$ that satisfies pure-DP. This mechanism, \Cref{alg:find-example-pure-DP}, is presented and analyzed in \Cref{sec:PREM_pure_DP}.
\begin{lemma} \label{lem:privacy_accuracy_find_example_pure_DP}
\Cref{alg:find-example-pure-DP} is $\eps$-DP and 
with probability at least $1 - \beta$, satisfies:
    \begin{itemize}[topsep=3pt,itemsep=-3pt]
    \item $\ind_{\cS}(\hbh) \geq \frac{1}{3} \ind_{\cY}(\hbh)$.
    \item If $(\theta,\cS)$ is the output of the algorithm, (see the pseudocode for the value of $\alpha_0$):
        \begin{itemize}[topsep=3pt,itemsep=-3pt]
            \item If $\theta = \equal$, for all $f\in \cF$, $(1 - \zeta) \cdot f(\bh^{\ast}|_{\cS}) - \alpha_0 \leq  f(\hbh|_{\cS}) \leq (1 + \zeta) \cdot  f(\bh^{\ast}|_{\cS}) + \alpha_0$.
            \item If $\theta = \plus$, $\ind_{\cS}(\bh^{\ast}) \geq (1 + \zeta) \cdot \ind_{\cS}(\hbh)$.
            \item If $\theta = \minus$, $\ind_{\cS}(\bh^{\ast})  \leq (1 - \zeta/2) \cdot \ind_{\cS}(\hbh)$.
        \end{itemize}
    \end{itemize}
\end{lemma}
We now proceed to $\PREM$. For brevity, we only emphasize the main aspects of this algorithm that require adaptation for the pure-DP case. Namely,
we make the simple observation that the MWU analysis works analogously to the approximate-DP case, with the caveat that for Claim \ref{claim:failure_PREM} to hold, we need $\tilde n_i>\alpha/4$ implies $n_i>\alpha/6$ with sufficiently high probability. This property depends on Laplace concentration \eqref{eqn:Laplace_concentration_MWU}, where for pure-DP we require instead
$\alpha \geq  O\Big(\frac{I}{\eps}\ln \frac{1}{\beta'} \Big).$
The rest of the proof works analogously, and noting that the bound $\alpha \geq O\Big(I\alpha_0\Big)=\tilde O\Big( \sqrt{\frac{n \log^3 n}{\zeta^2\eps}\log|\cX|\log \frac{|\cF|}{\beta} } \Big),$ is required, we conclude the following result.
\begin{theorem} \label{thm:efficient_pure_DP_UP}
For all $\eps \leq 1$, $\beta \in (0, 1)$, and $\zeta \in (0, \nicefrac12)$, 
there is an $\eps$-DP $(\zeta,\alpha,\beta,\cF)$-accurate synthetic data generator for any domain $\cX$ and query family $\cF$ with
\[ \alpha =\tilde O\Big( \sqrt{\frac{n \log^3 n}{\zeta^2\eps}\log|\cX|\log \frac{|\cF|}{\beta}} 
+\frac{\log n}{\eps}\log \frac{1}{\beta} \Big). \]
\end{theorem}

\section{Lower Bounds}

\label{sec:lower-bounds}

\subsection{Lower Bounds for Approximate-DP Algorithms}

The following lower bound shows that our approach is nearly optimal, up to polylogarithmic factors in $n,1/\delta$, and $\log |\cX|$. Note that the lower bounds we provide are for the expected additive accuracy of relatively accurate algorithms (more details in \Cref{def:expected_sample_accuracy}). These lower bounds apply to our algorithms, given that we provide high probability upper bounds for them, and via tail integration they enjoy in-expectation guarantees. Finally,
note that relative approximations are only nontrivial (compared to purely additive ones) when $\zeta\geq \alpha/n$. The proof of this result is deferred to \Cref{sec:pf_approx_DP_LB}.
\begin{theorem}  \label{thm:approx_DP_LB}
Let ${\cX}$ be a finite set and $k\in \N$, such that  
$|{\cX}| = \omega(\log k)$ and $k = \omega(\log|{\cX}|)$. 
Let $\frac{1}{|\cX|}\leq \delta<\frac{1}{\log|\cX|\log k}$. 
Let $\frac{2^{-\ln^{1/9}|{\cX}|}}{10}\leq  \frac{\alpha}{n} \leq \zeta<\frac{1}{10}$. Then there exists $\cF\subseteq\{0,1\}^{\cX}$ with $|\cF|=k$, such that if
${\cal A}:{\cX}^n\mapsto\R_+^{\cF}$ is an $(\eps,\delta)$-DP query-answering algorithm that is $(\zeta,\alpha,\cF)$-sample-accurate in expectation, then
\begin{align} 
n = \Omega\Big(\frac{\ln|{\cF}|\sqrt{\ln|{\cX}| \ln(1/\delta)}}{\eps \zeta^2}\Big) \quad \mbox{ and }\quad  
\alpha = \Omega\Big(\frac{\ln|{\cF}|\sqrt{\ln|{\cX}| \ln(1/\delta)}}{\eps \zeta}\Big). \label{eqn:LB_approx_DP}
\end{align}
In the case $\delta=0$, the corresponding lower bounds are
\[ n = \Omega\Big( \frac{\log|{\cF}| \log|\cX|}{\eps \zeta^2} \Big) \quad\mbox{ and }\quad  \alpha = \Omega\Big( \frac{\log|{\cF}| \log|\cX|}{\eps \zeta} \Big). \]
\end{theorem}

Note that the pure-DP lower bound does not prove near-optimality of our upper bounds, yet it shows a polylog dependence on $|{\cF}|$, $|\cX|$, and an inverse polynomial dependence on $\eps,\zeta$. Further understanding on the optimal rates with respect to $n$ is an interesting question for future work, which seems related to current barriers arising in the purely additive case \citep{DPorg-open-problem-optimal-query-release}. 

\section{Conclusion and Discussion} \label{sec:discussion}

We explore the question of privately answering linear queries with both relative and additive errors. Perhaps surprisingly, we give an approximate-DP algorithm which, if a constant relative error is allowed, can achieve the additive error that is polylogarithmic in $n, |\cX|, |\cF|, 1/\delta$. In comparison, without any relative error, the additive errors have to be polynomial in one of these parameters. We also show a nearly-matching lower bound for a large regime of parameters.

Perhaps the most obvious open question from our work is the pure-DP case. Our algorithm's additive error still has a dependency of $\sqrt{n}$. Is it possible to, similar to approximate-DP, achieve an additive error that is poly-logarithmic in $n, |\cX|, |\cF|, 1/\delta$ (for any small constant relative error $\zeta$)? We remark that the main challenge in obtaining this through our PREM framework lies in the SVT algorithm. Namely, to discover many queries above/below thresholds with pure-DP, the privacy budget grows with the number of discovered queries. 
If such a task can be accomplished without the increased budget, then applying PREM would result in the desired error bound. We note that, in a recent work, \cite{GhaziKRMS24} devises such a procedure, but only for \emph{above} threshold queries; extending their technique to \emph{below} threshold queries might answer this question.

Another research direction is to consider specific set of queries, such as $k$-marginal or graph cut queries. In fact, this direction was undertaken already by aforementioned previous work, e.g.,~\cite{EpastoMMMVZ23,Ghazi:2023} who considered range queries and hierarchical queries, respectively. Focusing on a specific query family may allow for better error guarantees, and more efficient algorithms. For the latter, we note once again that our algorithm runs in time polynomial in $n, |\cF|, |\cX|$. 
One might hope for an algorithm that runs in time polynomial in $n, |\cF|, \log|\cX|$; however, under cryptographic assumptions, this is known to be impossible for general family $\cF$ (and even for $k$-marginal queries)~\citep{UllmanV20}. However, such an efficient algorithm may exist for other specific query families, and this remains an intriguing direction for future research.

\bibliographystyle{plainnat}
\bibliography{refs}

\newpage

\appendix

\crefalias{section}{appendix} 

\section{Concentration Bounds}

\begin{lemma}[\cite{Mulzer:2018}] \label{lem:Mulzer}
Let $x_1,\ldots,x_n$ independent Bernoulli r.v.s with mean $p$. Then, for any $t>2enp$
\[ \Pr\Big[\sum_{i=1}^nx_i > t \Big] \leq 2^{-t}. \]
\end{lemma}

\begin{lemma}[Multiplicative Chernoff Bounds] \label{lem:multiplicative_Chernoff}
Let $(X_i)_{i\in[n]}$ be $\{0,1\}$-valued independent random variables, and let $X:=\sum_{i=1}^n X_i$. Then for any $\delta>0$
\[ \Pr\big[X>(1+\zeta)\mathbb{E}[X] \big]\leq \Big( \frac{\exp(\zeta)}{(1+\zeta)^{(1+\zeta)}} \Big)^{\mu}. \]
And for $0<\zeta<1$,
\[ \Pr\big[X\leq (1-\zeta)\mathbb{E}[X] \big]\leq \Big( \frac{\exp(-\zeta)}{(1-\zeta)^{(1-\zeta)}} \Big)^{\mu}. \]
\end{lemma}

\begin{proposition}[Concentration for Laplace random variables] \label{prop:Laplace_concentration}
Let $x\sim \Lap(b)$ (i.e.~have density $p(x)=\frac{1}{2b}e^{-|x|/b}$). Then, for all $0<\gamma<1$, it holds that
\[\textstyle \Pr\sq{|x|>b\ln\frac{1}{\gamma}} = \gamma.\]
\end{proposition}
\begin{proof}
This follows from direct integration,
\[ \Pr[|x|>\tau] = 2 \Pr[x > \tau] = \frac{1}{b} \int_{\tau}^{+\infty} e^{-x/b}dx = \int_{\tau/b} e^{-y} dy = e^{-\tau/b}.\]    
\end{proof}

\newcommand{\bin}{\mathrm{bin}}
\newcommand{\rval}{\mathrm{real}}
\newcommand{\bha}{\bh^{\ast}}

\section{From Binary Queries to Real-Valued Queries}

\label{app:real-v-bin}

We refer to queries of the form $f: \cX \to \{0, 1\}$ as \emph{binary queries} and those of the form $f: \cX \to [0, 1]$ as \emph{real-valued queries}. (Sometimes these are referred to as \emph{counting queries} and \emph{linear queries}, respectively, in literature.) In this section, we prove the following theorem, which translates the accuracy guarantee for binary query families to real-valued query families in a black box manner with almost no additional overhead.

\begin{theorem}\label{thm:bin-to-real}
Suppose that for any binary query family $\cF_{\bin} \subseteq \{0, 1\}^{\cX}$, there exists a synthetic data generator $\cB_{\bin} : \cX^* \to \cZ_{\ge 0}^{\cX}$ that is $(\zeta, \alpha_{\bin}, \beta, \cF_{\bin})$-accurate, for some error bound $\alpha_{\bin} := \alpha_{\bin}\prn{\eps, \delta, n, |\cF_{\bin}|, |\cX|, \zeta, \beta}$. Then, for any $\zeta \leq 0.5$, there is a synthetic data generator $\cB_{\rval} : \cX^* \to \cZ_{\ge 0}^{\cX}$ for any real-valued query family $\cF_{\rval} \subseteq [0, 1]^{\cX}$ such that it is $(\zeta, \alpha_{\rval}, \beta, \cF_{\rval})$-accurate for
\[\textstyle
\alpha_{\rval} := O\left(1 + \alpha_{\bin}\left(\eps, \delta, n, O\left(\frac{|\cF_{\rval}| \cdot \log n}{\zeta}\right), |\cX|, \frac\zeta{10}, \beta\right)\right).
\]
\end{theorem}

\begin{proof}
In the argument below, we will translate each real-valued query to binary queries using ``thresholding'' construction. For this purpose, for every real-valued query $f: \cX \to [0, 1]$ and $\tau \in \R$, let $f_{\geq \tau}: \cX \to \{0, 1\}$ be defined by
\begin{align*}
f_{\geq \tau}(x) :=
\begin{cases}
1 & \text{ if } f(x) \geq \tau, \\
0 & \text{ otherwise,}
\end{cases}
\qquad \forall x \in \cX.
\end{align*}

Let $\cF_{\rval}$ be any family of $k$ real-valued queries. The algorithm $\cB_{\rval}$ works as follows:
\begin{itemize}
\item Let $\zeta' = 0.1\zeta, L := \lceil \log_{1+\zeta'} n \rceil + 1$ and, for all $i \in \{0, \dots, L\}$, let $\tau_i := \frac{1}{(1 + \zeta')^i}$. For convenience, also let $\tau_{L+1} = 0$.
\item Let $\cF_{\bin}$ be the set of binary queries defined by $\cF_{\bin} := \{f_{\geq \tau_i} \mid i \in \{0, \dots, L+1\}, f \in \cF_{\rval}\}$.
\item Run $\cB_{\bin}$ with the query family $\cF_{\bin}$ and relative approximation parameter $\zeta'$ (with the other parameters remaining the same) to produce a synthetic data $\hbh$.
\item Output $\hbh$.
\end{itemize}
Note that $|\cF_{\bin}| = |\cF_{\rval}| \cdot \left(L + 1\right) = O\left(\frac{k \log n}{\zeta}\right)$. Thus, $\cB_{\bin}$ is $(\zeta', \alpha', \beta, \cF_{\bin})$-accurate for $\alpha' = \alpha\left(\eps, \delta, n, O\left(\frac{k \log n}{\zeta}\right), |\cX|, \zeta', \beta\right)$. That is, with probability $1 - \beta$, we have that for all $f \in \cF$ and $i \in \{0, \ldots, L+1\}$,
\begin{align} \label{eq:acc-bin}
(1 - \zeta') \cdot f_{\geq \tau_i}(\bh^{\ast}) - \alpha' \leq  f_{\geq \tau_i}(\hbh) \leq (1 + \zeta') \cdot f_{\geq \tau_i}(\bh^{\ast}) + \alpha'.
\end{align}
We will assume that this event holds for the rest of the proof and analyze the accuracy of $\cB_{\rval}$. Consider any $f \in \cF$. Let $\tf: \cX \to \R$ be defined as $\tf(x) := \sum_{i=0}^L (\tau_i - \tau_{i + 1}) f_{\geq \tau_i}(x)$. We begin by proving the following lemma:
\begin{lemma}
For any $\bh \in \R_{\geq 0}^{\cX}$, we have 
\begin{align} \label{eq:real-to-bin-sandwich}
\tf(\bh) \leq  f(\bh) \leq (1 + \zeta') \tf(\bh) + \tau_L \cdot\|\bh\|_1.
\end{align}
\end{lemma}
\begin{proof}
For every $x \in \cX$, let $i(x) \in \{0, \dots, L+1\}$ denote the smallest index $i$ such that $f(x) \geq \tau_i$. We can bound $f(\bh)$ from below as follows:
\begin{align*}
f(\bh) &\textstyle= \sum_{x \in \cX} h(x) f(x) \\
&\textstyle\geq \sum_{x \in \cX} h(x) \tau_{i(x)} \\
&\textstyle= \sum_{x \in \cX} h(x) \sum_{i = i(x)}^{L} (\tau_{i(x)} - \tau_{i(x) + 1}) \\
&\textstyle= \sum_{i=0}^L (\tau_i - \tau_{i + 1}) \sum_{x \in \cX} h(x) \cdot \ind[i \geq i(x)] \\
&\textstyle= \sum_{i=0}^L (\tau_i - \tau_{i + 1}) \cdot f_{\geq \tau_i}(\bh) \\
&\textstyle= \tf(\bh).
\end{align*}
For the upper bound, note that our definition of $i(x)$ implies that $f(x) \leq (1 + \zeta') \tau_{i(x)} + \tau_L$. This implies that
\begin{align*}
f(\bh) &\textstyle~=~ \sum_{x \in \cX} h(x) f(x)
\textstyle~\leq~ \sum_{x \in \cX} h(x) \left((1 + \zeta') \tau_{i(x)} + \tau_L\right)
\textstyle~=~ (1 + \zeta') \tf(\bh) + \tau_L \cdot \|\bh\|_1.
\end{align*}
\end{proof}
We now continue with our accuracy analysis using the above lemma together with \eqref{eq:acc-bin}. We note here that $\|\bh\|_1 = \langle\bh, f_{\geq \tau_{L+1}}\rangle$. Note also that, since $\tf$ is a convex combination of $f_{\geq \tau_i}$'s, \eqref{eq:acc-bin} implies that
\begin{align} \label{eq:acc-lin-comb}
(1 - \zeta') \cdot \tf(\bh^{\ast}) - \alpha' \leq  \tf(\hbh) \leq (1 + \zeta') \cdot \tf(\bh^{\ast}) + \alpha'.
\end{align}
Also note that $\tau_{L} \cdot n \le 1$. With this ready, we can upper bound $f(\hbh)$ as follows:
\begin{align*}
f(\hbh) &\overset{\eqref{eq:real-to-bin-sandwich}}{\leq}  (1 + \zeta') \tf(\hbh) + \tau_L \cdot \|\hbh\|_1 \\
&\!\!\!\overset{\eqref{eq:acc-lin-comb}, \eqref{eq:acc-bin}}{\leq}  (1 + \zeta') \left((1 + \zeta') \cdot \tf(\bha) + \alpha'\right) + \tau_L \cdot \left((1 + \zeta') \|\bha\|_1 + \alpha'\right) \\
&\leq (1 + \zeta) \tf(\bha) + 3\alpha' + 1 \\
&\overset{\eqref{eq:real-to-bin-sandwich}}{\leq} (1 + \zeta) f(\bha) + 3\alpha' + 1.
\end{align*}
Finally, it can be similarly lower bounded as
\begin{align*}
f(\hbh) &\overset{\eqref{eq:real-to-bin-sandwich}}{\geq}  \tf(\hbh) \\
&\overset{\eqref{eq:acc-lin-comb}}{\geq} (1 - \zeta') \cdot \tf(\bh^{\ast}) - \alpha' \\
&\overset{\eqref{eq:real-to-bin-sandwich}}{\geq} (1 - \zeta') \cdot \left(\frac{1}{1 + \zeta'} \cdot (f(\bh^{\ast}) - \tau_L \cdot n) \right) - \alpha' \\
&\geq (1 - \zeta) f(\bh^{\ast}) - \alpha' - 1. 
\end{align*}
\end{proof} 
\section{Further Background on DP and the Target Charging Technique}

\subsection{Basic Mechanisms and Properties in DP}

\label{sec:basic_DP_properties}

A key property of DP is its closure under adaptive composition; namely for mechanisms $\cM_1$ and $\cM_2$, the adaptive composition refers to the mechanism that outputs $(a_1 = \cM_1(\bh), a_2 = \cM_2(a_1, \bh))$ (composition of $k$ mechanisms can be defined analogously).

\begin{proposition}\label{prop:dp-composition}
The adaptive composition of $k$ mechanisms, each satisfying $(\eps, \delta)$-DP, satisfies $(k\eps, k\delta)$-DP. Furthermore, for any $0<\delta'\leq 1$ it also satisfies $(\eps',k\delta+\delta')$-DP, where $\eps'=\eps[\sqrt{2k\ln(1/\delta')}+k\frac{e^{\eps}-1}{e^{\eps}+1}]$.
\end{proposition}

The Laplace distribution with parameter $b$ is supported over $\R$ and has the probability density function $P_{\Lap(b)}(x) = \frac{1}{2b} e^{-|x|/b}$. The Laplace mechanism\footnote{
Since we only use the Laplace mechanism to estimate functions which take integer values, we could also use the discrete Laplace mechanism that is supported over $\Z$ with probability mass function $P_{\mathrm{DLap}(a)}(x) = \frac{e^a-1}{e^a+1} \cdot e^{-a|x|}$.
} is defined as follows.

\begin{proposition}[Laplace mechanism]\label{prop:laplace}
For any function $g: \Z_{\ge 0}^{\cX} \to \R$ such that $|g(\bh) - g(\bh')| \le 1$ for any adjacent $\bh$, $\bh'$, the mechanism that returns $g(\bh) + \Lap(1/\eps)$ satisfies $\eps$-DP.
\end{proposition}

\subsection{$\OutsideIntervalMonitor$ with Pure DP}\label{sec:OIM_pure_DP}

We provide a pure DP version of the $\OutsideIntervalMonitor$. For this, we resort on the classical {\em sparse vector technique} for Above Threshold queries.

\begin{algorithm}[t]
\caption{$\OutsideIntervalMonitor^{\mathrm{pure}}_{\bh, a, \Gamma, \cY}$ (Pure-DP version)}
\label{alg:outside-thresholds-pure-dp}

\textbf{Initialization:}
{\footnotesize $\bullet$ } $\bh \in \Z_{\ge 0}^{\cX}$ : private histogram \qquad \qquad
{\footnotesize $\bullet$ } $a > 0$ : noise parameter \\ 
\phantom{\textbf{Initialization:}\,}
{\footnotesize $\bullet$ } $\Gamma > 0$ : maximum number of rounds when $\Above$ or $\Below$ is returned\\
\phantom{\textbf{Initialization:}\,}
{\footnotesize $\bullet$ } $\Sactive \gets \cY$ : initial active set (for $\cY \subseteq \cX$)\\
\phantom{\textbf{Initialization:}\,}
{\footnotesize $\bullet$ } $\counter \gets 0$ \qquad \qquad \qquad\qquad \qquad \qquad 
{\footnotesize $\bullet$ } $\eta \sim \Lap(2/a)$ : noise for threshold\\[-2mm]

\textbf{On input} $(f: \cX \to \{0, 1\}, \tau_{\ell},\tau_u\in\R_{\geq 0})$ {\bf :}\\[-2mm]

\If{{\em $\counter \geq  \Gamma$}}{
    \textbf{halt}
}

\If{$f(\bh|_{\Sactive}) + \Lap(4/a)\geq \tau_u+\eta$}{
    $\counter\gets \counter+1$\\
    $\eta\sim\Lap(2/a)$\\
    $\Sactive \gets \Sactive \smallsetminus f^{-1}(1)$\algcomment{Update state for future rounds before returning}\\
    \Return $\Above$\;
}
\ElseIf{$f(\bh|_{\Sactive}) + \Lap(4/a)\leq \tau_{\ell}-\eta$}{
    $\counter\gets \counter+1$\\
    $\eta\sim\Lap(2/a)$\\
    $\Sactive \gets \Sactive \smallsetminus f^{-1}(1)$\algcomment{Update state for future rounds before returning}\\
    \Return $\Below$\;
}
\Else{
    \Return $\Between$\;
}
\end{algorithm}

\begin{proposition}[$\OutsideIntervalMonitor$ (Pure-DP) guarantees]\label{prop:oim-pure-dp}
For any integer $R > 0$, the mechanism
$\OutsideIntervalMonitor^{\mathrm{pure}}_{\bh, a, \Gamma, \cY}$ (\Cref{alg:outside-thresholds-pure-dp}), after $R$ rounds of queries, satisfies $\eps$-DP for $\eps = a\Gamma$.
Let $\beta \in (0, 1)$, $C := \frac6a \log(\frac{R}{\beta})$, and $\Sactive$ be the state of the set before the query is made. Then with probability at least $1-\beta$, on any query $(f, \tau_{\ell}, \tau_{u})$,
\begin{itemize}[topsep=3pt,itemsep=-3pt]
    \item If $\Between$ is returned then $\tau_{\ell} - C \le f(\bh|_{\Sactive}) \le \tau_u + C$.
    \item If $\Above$ is returned then $f(\bh|_{\Sactive}) \ge \tau_u - C$.
    \item If $\Below$ is returned then $f(\bh|_{\Sactive}) \le \tau_{\ell} + C$.
\end{itemize}
\end{proposition}

\begin{proof}
The privacy of the algorithm follows by reduction from the sparse vector technique with \textsc{AboveThreshold} queries with a fixed threshold (e.g.~\cite{DworkR14}).
Namely, every $\Above$ query in \Cref{alg:outside-thresholds-pure-dp} corresponds to an \textsc{AboveThreshold} query with threshold 0 for the function
\[ g(x)=f(x)-\frac{\tau_u}{|\Sactive|},\]
whereas the $\Below$ query in the algorithm correspond to \textsc{AboveThreshold} query with threshold 0 for the function
\[ g(x)=\frac{\tau_{\ell}}{|\Sactive|}-f(x).\]
This reduction implies that the standard analysis of SVT with \textsc{AboveThreshold} queries yields the privacy guarantee of $(a\Gamma)$-DP.

The second part of the result follows from standard concentration of Laplace random variables with a union bound over the $R$ rounds (see \Cref{prop:Laplace_concentration}).
\end{proof}

\subsection{The Target Charging Technique}

\label{app:TCT}

We present here some of the technical background behind the target charging technique and its application to the $\OutsideIntervalMonitor$ with individual privacy accounting, \Cref{alg:outside-thresholds}.

We recall that two probability distributions $\mathbb{P},\mathbb{Q}$ are $(\eps,\delta)$-indistinguishable, denoted $A\stackrel{\eps,\delta}{\approx} B$ if for any event $E$, $\mathbb{P}[E]\leq e^{\eps}\mathbb{Q}[E]+\delta$. Hence, an algorithm $\cA$ is $(\eps,\delta)$-DP iff for all pair of neighboring datasets $D^0,D^1$, $\cA(D^0)\stackrel{\eps,\delta}{\approx}\cA(D^1)$.

\begin{definition}[\boldmath $q$-target with $(\eps,\delta)$]
Given $0< q\leq 1$ and $\eps\geq 0$, $0\leq \delta\leq 1$, 
let $Z^0,Z^1$ be probability distributions supported on a set $\cY$. We say that $\top$ is a $q$-target of $(Z^0,Z^1)$ with $(\eps,\delta)$ if there exists $p\in[0,1]$ and five probability distributions $C,B^b,E^b$ (where $b\in\{0,1\}$) such that $Z^0$ and $Z^1$ can be written as the mixtures
\begin{align*}
Z^0 &= (1-\delta)(p C+(1-p)B^0)+\delta E^0\\
Z^1 &= (1-\delta)(p C+(1-p)B^1)+\delta E^1,
\end{align*}
where $B^0$ and $B^1$ are $\eps$-indisginguishable, and $\min\{\mathbb{P}[B^0\in \top],\, \mathbb{P}[B^1\in \top]\} \geq q$.

Let $\cA:\cX^n \to \cY$ be a randomized algorithm. We say that $\top\subseteq \cY$ is a $q$-target of $\cA$ with $(\eps,\delta)$, if for any pair of neighboring datasets $D^0,D^1$, $\top$ is a $q$-target of $(\cA(D^0),\cA(D^1))$ with $(\eps,\delta)$.
\end{definition}

We illustrate the utility of the concept of $q$-targets with an example. 
For threshold-style mechanisms, outcomes are only useful when they reach a target set. Otherwise, it suffices to return a ``prior'' output $\perp$. 

\begin{lemma}[Lemma C.1 in \cite{Cohen:2023}]
\label{lem:target}
Let $\cA:\cX\mapsto\cY\cup\{\perp\}$, where $\perp\notin\cY$, be an $(\eps,\delta)$-DP algorithm. Then $\cY$ is a $\frac{1}{1+\exp(\eps)}$-target of $\cA$ with $(\eps,\delta)$.
\end{lemma}

The Target Charging Technique ensures that for adaptive sequences of private mechanisms with corresponding $q$-targets, one can release only the outcomes that reach the targets, paying for privacy accumulation only accounting for the steps where targets are reached.

\begin{theorem}[Simplification of Theorem B.4 in \cite{Cohen:2023}] \label{thm:TCT}
Let $(\cA_i,\top_i)_{i\geq 1}$ a sequence of pairs of an $\eps$-DP mechanism $\cA_i$ and set $\top_i$ that is a $q$-target of $\cA_i$ with $\eps$. Let $\tau\geq 1$.
Then, the transcript that only releases the computations where the target is reached until $\tau$ target hits occur satisfies $(\eps',\delta')$-DP with 
\[\eps'=\frac{2\tau}{q}\eps,\qquad \delta'\leq \exp(-\tau/4).\]
\end{theorem}

We conclude with the privacy analysis of $\OutsideIntervalMonitor$, which closely follows the one  for $\SVT$ in \cite{Cohen:2023}. In principle, for individual privacy accounting we would only need $\tau=1$ in the target hit counter; however, to get the desired privacy parameters from the previous result, we will choose $\tau=\Omega(\log(1/\delta))$.

\begin{proofof}{\Cref{prop:oim-approx-dp}}
For simplicity, we prove for add/remove neighboring datasets a DP guarantee. To turn this into a replace-one neighboring condition it suffices to compose the previous argument twice. 

Consider datasets $D$, $D'=D\circ\{x\}$. 
We perform a simulation-based analysis \citep{CohenSTOC:23,Cohen:2023}. Note first that the updates on the active set can be obtained by post-processing (only based on the query and the output), and therefore it suffices to focus on the privacy of the returned output. 
If any query $f$ is such that $f(x)=0$, the answers on both datasets are identically distributed, and no interaction with the data holder is needed. By contrast, if $f(x)=1$, the simulator can run the Laplace mechanism and provide its output to the simulator. 
Furthermore, upon an $\Above$ or $\Below$ output in this latter case, $x$ is removed from the active set, and then the outcome distributions are identical for any further query, therefore no further access to the data holder is needed. 

To conclude, by \Cref{lem:target}, $\Between$ is a $q$-target of the Laplace mechanism with $a$, where $q=1/[1+\exp(a)]=\Omega(1)$ (since $0\leq a\leq 1$). By \Cref{thm:TCT}, the entire transcript is $(\eps,\delta)$-DP for $\eps=O(a\log(1/\delta))$.
\end{proofof}

\section{Pure DP Upper Bounds by the Exponential Mechanism}

\label{app:exponential_mechanism}

We show in this section that the exponential mechanism for a suitable score function and histogram support size provides rates which are slightly tighter than those encountered by $\PREM$ (by a polylog$n$ factor), at the expense of running time that is quasi-polynomial in $n,|\cX|,|\cF|$. 

\begin{theorem} \label{thm:Pure_DP_Upper_Bound}
There exists an $\eps$-DP mechanism such that it computes an $(\zeta,\alpha,\beta)$-accurate histogram, where
\[ \alpha \leq 2\sqrt{\frac{16(1+\zeta)e n}{\eps \zeta^2}\ln(4|{\cal F}|)\ln|{\cal X}|}+\frac{2(1+\zeta)}{\eps}\ln\frac{1}{\beta}. \]
\end{theorem}

Instrumental to the analysis of our algorithm is the existence of sparse datasets which relatively approximate arbitrary histograms. This is a multiplicative/additive counterpart of an additive version proved in \citep{BlumLR13}.

\begin{proposition} \label{prop:existence_approx_histogram}
Let $0<\zeta<1$ and $\alpha'\in(0,n]$ and ${\cal F}$ a set of linear queries. Then, for any histogram $\bh^{\ast}\in\R_+^{\cX}$ with $\|\bh^{\ast}\|_1=n$, there exists a (rescaled) histogram $\tbh\in\N_+^{\cX}$ with $\|\tbh\|_1=n$ and support size 
$k=\frac{8e}{\zeta^2}\frac{n}{\alpha'}\ln(4|{\cal F}|)$ such that for all $f\in {\cal F}$
\begin{equation} \label{eqn:approx_multip}
(1-\zeta)f(\bh^{\ast}) - \alpha' \leq f(\tbh) \leq (1+\zeta) f(\bh^{\ast})+\alpha'.
\end{equation}
\end{proposition}

\begin{proof} Let $0<\beta<1$ to be determined. Let $\tbh\in\mathbb{N}^{|\cX|}$ obtained as the histogram obtained from $k$ samples with replacement from the probability distribution $\bh^{\ast}/n\in \Delta_{\cX}$, and repeating each sample $n/k$ times. Denote those samples as $x_1,\ldots,x_k$ and note that $f(\tbh)=\frac{n}{k}\sum_{i=1}^k f(x_i)$, hence $f(\tbh)$ is a (scaled) sum of independent Bernoulli, and furthermore $\mathbb{E}[f(\tbh)]=n\mathbb{E}[f(x_1)]=f(\bh^{\ast})$.

We divide the analysis into bounding the probabilities for both the upper and lower multiplicative/additive  bounds.

\noindent\textbf{Upper bound.} We analyze the deviations by cases. First, if $2ef(\bh^{\ast}) < \alpha'$, then by \Cref{lem:Mulzer},
\[
\mathbb{P}[ f(\tbh) >(1+\zeta)f(\bh^{\ast})+\alpha'] \leq \mathbb{P}[ f(\tbh) >\alpha'] 
= \mathbb{P}\Big[ \sum_{i=1}^k f(x_i) >\frac{k\alpha'}{n} \Big] 
\leq 2^{-\frac{k\alpha'}{n}} \leq \frac{\beta}{2|{\cal F}|}.
\]
Hence, we can choose $k = \frac{n}{\alpha'}\log_2\big(\frac{2|{\cal F}|}{\beta}\big)$ to make the probability as small as claimed.

In the other case, where $2ef(\bh^{\ast}) \geq \alpha'$, we have by the multiplicative Chernoff bound (\Cref{lem:multiplicative_Chernoff}),
\[
\mathbb{P}[f(\tbh) >(1+\zeta)f(\bh^{\ast})+\alpha'] = 
\mathbb{P}\Big[\sum_{i=1}^k f(x_i) >(1+\zeta)k\frac{f(\bh^{\ast})}{n}+\frac{k\alpha'}{n}\Big] \\
\leq \Big(\frac{\exp(\zeta)}{(1+\zeta)^{(1+\zeta)}}\Big)^{\frac{kf(\bh^{\ast})}{n}}\leq \frac{\beta}{2|{\cal F}|}.
\]
This leads to a bound on $k=\frac{4n}{\zeta^2f(\bh^{\ast})}\ln\big(\frac{2|{\cal F}|}{\beta}\big)\leq \frac{8en}{\zeta^2 \alpha'}\ln\big(\frac{2|{\cal F}|}{\beta}\big)$.\\

\noindent\textbf{Lower Bound.} Let us first consider the case when $\alpha'>f(\bh^{\ast})$. We note that in this case $(1-\zeta)f(\bh^{\ast})-\alpha'<0,$ 
hence $\mathbb{P}[f(\tbh) < (1-\zeta)f(\bh^{\ast})-\alpha']=0$.

For the other case, $\alpha'\leq f(\bh^{\ast})$, using the lower-tail  multiplicative Chernoff bound (\Cref{lem:multiplicative_Chernoff}),
\begin{align*}
\mathbb{P}[f(\tbh) < (1-\zeta)f(\bh^{\ast})-\alpha'] &= \mathbb{P}\Big[\sum_{i=1}^k f(x_i) <(1-\zeta)k\frac{f(\bh^{\ast})}{n}-\frac{k\alpha'}{n}\Big] 
\leq \Big(\frac{\exp(-\zeta)}{(1-\zeta)^{(1-\zeta)}}\Big)^{\frac{kf(\bh^{\ast})}{n}} \\
&\leq \frac{\beta}{2|{\cal F}|},
\end{align*}
where the last step follows for $k=\frac{2}{\zeta^2}\frac{n}{f(\bh^{\ast})}\ln\big(\frac{2|{\cal F}|}{\beta}\big)\leq \frac{2}{\zeta^2}\frac{n}{\alpha'}\ln\big(\frac{2|{\cal F}|}{\beta}\big)$. 

Finally, using a union bound over the $2|{\cal F}|$ events needed for \eqref{eqn:approx_multip}, we have that with probability $1-\beta$ this approximation holds. Letting $\beta=1/2$ yields the result.
\end{proof}

\subsubsection{Exponential Mechanism for Relative Approximation}

With the above existential result, the pure-DP mechanism is straightforward. 

\begin{proofof}{\Cref{thm:Pure_DP_Upper_Bound}}
The proposed algorithm is running the exponential mechanism on (rescaled) histograms of size $k$ (where $k$ is as proposed in \Cref{prop:existence_approx_histogram}) for the score function
\[ s(\bh) = -\max_{f\in {\cal F}} \Big\{f(\bh) - (1+\zeta)f(\bh^{\ast}),(1-\zeta)f(\bh^{\ast})-f(\bh )\Big\}. \]
Since the sensitivity of this score function is bounded by $(1+\zeta)$, it suffices to set the temperature parameter for the exponential mechanism as $\frac{\eps}{2(1+\zeta)}$.

Notice that the space of such histograms ${\cal H}_k$ has cardinality bounded by $|{\cal X}|^{k}$. In particular, its logarithm is $\ln|{\cal H}_k|\leq \frac{8e}{\zeta^2}\frac{n}{\alpha'}\ln(4|{\cal F}|)\ln|{\cal X}|$. By the accuracy of the exponential mechanism, for $0<\beta<1$
\[ \mathbb{P}\Big[s(\hat\bh)-\max_{\bh}s(\bh) < -\frac{2(1+\zeta)[\log|{\cal H}_k|+\ln\frac{1}{\beta}]}{\varepsilon} \Big] \leq \beta.\]
Under the complement of the event above (which happens with probability $1-\beta$)
for all $f\in {\cal F}$ 
\[
\hspace{-0.3cm}f(\hbh)-(1+\zeta)f(\bh^{\ast}) \leq \min_{\bh}s(\bh) + \frac{2(1+\zeta)\log|{\cal H}_k|+\log\frac{1}{\beta}}{\varepsilon} \\
\leq \alpha'+ \frac{2(1+\zeta)\log(|{\cal H}_k|/\beta)}{\varepsilon}.
\]
And analogously,
\[
(1-\zeta)f(\bh^{\ast}) - f(\hbh)
\leq \min_{\bh}s(\bh)+ \frac{2(1+\zeta)\log(|{\cal H}_k|/\beta)}{\varepsilon} 
\leq \alpha'+ \frac{2(1+\zeta)\log(|{\cal H}_k|/\beta)}{\varepsilon}.
\]
Next, we have 
\[
\alpha'+ \frac{2(1+\zeta)[\log|{\cal H}_k|+\ln(1/\beta])}{\eps} 
=\alpha'+\frac{2(1+\zeta)}{\eps}\Big[ \frac{8e}{\zeta^2}\frac{n}{\alpha'}\ln(4|{\cal F}|)\ln|{\cal X}|+\ln\frac{1}{\beta}\Big].
\]
Hence, selecting $\alpha'=\sqrt{\frac{2(1+\zeta)8e n}{\eps \zeta^2}\ln(4|{\cal F}|)\ln|{\cal X}|}$ yields the result.
\end{proofof}
\section{\boldmath Analysis of $\PREM$ under Pure DP}
\label{sec:PREM_pure_DP}

\subsection{Find Margin Example with Pure DP}

We present and analyze the pure-DP version of \textsc{Find Margin Example} \Cref{alg:find-example-pure-DP}.

\begin{algorithm}[t]
    \caption{\textsc{FindMarginExample}$_{\cF,\eps,0,\beta,\zeta}$ (\textsc{Pure DP Version})}
    \label{alg:find-example-pure-DP}
    \begin{algorithmic}
        \item[\textbf{Parameters:}] 
        {\footnotesize $\bullet$ } privacy parameter $\eps > 0$; 
            \qquad\qquad\qquad\quad
            {\footnotesize $\bullet$ } confidence parameter $0 < \beta < 1$;\\
            \qquad\qquad\,\,\, {\footnotesize $\bullet$ } approximation factor $0 < \zeta < 1 / 2$;
            \quad \quad \quad
            {\footnotesize $\bullet$ } set of counting queries $\cF\subseteq \{0, 1\}^{\cX}$;\\
            \qquad\qquad\,\,\,
            {\footnotesize $\bullet$ } active set $\cY\subseteq \cX$; and
            \qquad \qquad\qquad\qquad
            {\footnotesize $\bullet$ }public estimate histogram $\hbh\in \R_{\geq 0}^{\cX}$
        \item[\textbf{Input:}] input histogram $\bh^{\ast} \in \N_{\geq 0}^{\cX}$ with $\|\bh^{\ast}\|_1=n$.\\
        $\cS^{\plus}, \cS^{\minus} \gets \emptyset$, $\cS_{\acti} \gets \cY$ \;\\
        $\alpha_0\gets \sqrt{\frac{48 n\log n}{\zeta^2\eps}\log|\cX|\log\big(\frac{|\cF|}{\beta}\big)}$\; \\
        $\Gamma\gets 8n/\alpha_0$\;\\
        $a\gets\eps/\Gamma$\; \\
        $\cF_{\acti}\gets \cF$\; \\ 
        $\OIM\gets\OutsideIntervalMonitor^{\mathrm{pure}}_{\bh^{\ast},a,\Gamma,\cY}$ \hfill \algcomment{Initialization}\\
        \Repeat{\textsc{Accurate}}{
            \textsc{Accurate} $\gets$ \textsc{True}
            
            \For{$f\in\cF_{\acti}$}{
                $\tau_u \gets \frac{1}{(1-\zeta)}\big[ \langle \hbh|_{\cS_{\acti}},f\rangle+\frac{\alpha_0}{2}\big]$ and $\tau_{\ell} \gets \frac{1}{(1+\zeta)}\big[\langle \hbh|_{\cS_{\acti}},f\rangle-\frac{\alpha_0}{2}\big]$

                $s \gets \OIM(f, \tau_{\ell}, \tau_u)$

                \If{$s \in \{\textsc{Above}, \textsc{Below}\}$}{
                    \If{$s = \textsc{Above}$}{
                        $\cS^{+}\gets\cS^{+}\cup(\cS_{\acti}\cap f^{-1}(1))$\;
                    }
                    \Else {
                        $\cS^{-}\gets\cS^{-}\cup(\cS_{\acti}\cap f^{-1}(1))$\;
                    }
                    
                    $\cS_{\acti}\gets \cS_{\acti} \smallsetminus f^{-1}(1)$\; \algcomment{Identical to $\cS_{\acti}$ maintained in state of $\OIM$.}\\
                    $\cF_{\acti}\gets \cF_{\acti}\smallsetminus \{f\}$\;\\
                    \textsc{Accurate} $\gets$  \textsc{False}\;
                }
            }
        }
        \Return{$(\theta,\cS)\gets \begin{cases}
            (\plus,\Splus) & \text{\em if } \ind_{\Splus}(\hbh) \ge \max\{ \ind_{\Sminus}(\hbh), \ind_{\Sactive}(\hbh) \} \\
            (\minus,\Sminus) & \text{\em if } \ind_{\Sminus}(\hbh \ge \max\{ \ind_{\Splus}(\hbh), \ind_{\Sactive}(\hbh) \}  \\
            (\equal, \Sactive) & \text{\em if } \ind_{\Sactive}(\hbh) \ge \max\{ \ind_{\Sminus}(\hbh), \ind_{\Splus}(\hbh) \} 
            \end{cases}$
        }
    \end{algorithmic}
\end{algorithm}

\begin{proofof}{\Cref{lem:privacy_accuracy_find_example_pure_DP}}
The privacy of the algorithm reduces to that of the $\OutsideIntervalMonitor$, as it is the only part of the algorithm that uses the private data. Hence, by \Cref{prop:oim-pure-dp}, \Cref{alg:find-example-pure-DP} satisfies $\eps$-DP.

For the accuracy, by \Cref{prop:oim-pure-dp}, we have that with probability $1-\beta$ all queries are accurate up to $C=\frac1a\log(\frac{|\cF|^2}{\beta})\big)$. In particular, under \textsc{Above} answers, we have that the $\ell_1$-norm of $\bh^{\ast}|_{\cS_{\acti}}$ decreases at least by
\[ f(\bh^{\ast}|_{\cS_{\acti}}) \geq \frac{1}{2(1-\zeta)}\alpha_0-C\geq \frac{\alpha_0}{4}.  \]
Similarly, under \textsc{Below} answers, we have that the $\ell_1$-norm of $\hbh|_{\cS_{\acti}}$ decreases at least by
\[ f(\hbh|_{\cS_{\acti}}) \geq \frac{\alpha_0}{2}-(1+\zeta)C\geq \frac{\alpha_0}{4}. \]
Therefore, under the event above, there cannot be more than $(4n/\alpha_0)$ \textsc{Above} or \textsc{Below} answers  before termination. Since $\Gamma=8n/\alpha_0$, the algorithm does not fail under this event.

Upon non-failure, this algorithm operates analogously to \Cref{alg:find-example}, hence by \Cref{lem:bad-margin} the conclusions follow.
\end{proofof}

\section{Additive Approximation Lower Bounds for Differentially Private Query Answering}

\label{app:LB_proofs}

We provide a lower bound for DP query answering under purely-additive approximation error \citep{Lyu:2024}. We only slightly modify the construction to fit within the (nonnegative) counting queries setting we consider in this work.

The lower bounds we use are for the statistical setting which we now introduce. 
\begin{definition}[In-Expectation Population Accuracy]
\label{def:expected_population_accuracy}
We say that a randomized algorithm ${\cal A}:{\cal X}^n\mapsto \mathbb{R}^{\cal F}$ is $(\alpha,\cF)$-population-accurate in expectation if  for any probability distribution $\mathbb{P}$ supported on ${\cal X}$,
\[ \mathbb{E}_{\cA}\mathbb{E}_{D\sim \mathbb{P}^n} \Big[\max_{f\in {\cal F}}\big|{\cal A}_f(D)-\mathbb{E}_{x\sim\mathbb{P}}[f(x)]\big|\Big] \leq \alpha. \]
\end{definition}

\begin{theorem}[\cite{Lyu:2024}]
\label{thm:LB_additive_query}
Let ${\cal X}$ be a finite set, and $k>0$ such that $k = \omega(\ln|{\cal X}|)$ and $|X|=\omega(\ln k)$. 
There exists ${\cal F}\subseteq\{0,1\}^{\cal X}$ with $|{\cal F}|=k$ such that, for every $0<\eps<1$, $2^{-\log^{1/9}|{\cal X}|}<\alpha<1$ and $\frac{1}{|{\cal X}|} < \delta<\frac{1}{\log|{\cal X}| \log k}$, any ${\cal A}:{\cal X}^n\mapsto \R^{\cal F}$ which is $(\alpha,\cF)$-population-accurate in expectation and $(\eps,\delta)$-DP algorithm requires
\[ n=\Omega\Big( \frac{\log|{\cal F}|\sqrt{\log|{\cal X}|\log(1/\delta)}}{\eps\alpha^2} \Big). \]
\end{theorem}

This result follows from Theorem 8 in \cite{Lyu:2024}, where their construction uses query matrices $A\in \{-1,+1\}^{{\cal F}\times{\cal X}}$. It suffices to use the affine transformation
$A'=T(A)\in\{0,1\}^{{\cal F}\times{\cal X}}$, where $T$ applies coordinate-wise the operation $T(a_{ij})=\frac12(1+a_{ij})$, and  concluding by noting that $\alpha$-accuracy over the original dataset is equivalent to $\alpha/2$-accuracy over the image space.

We conclude by noting that in the context of query-answering, sample accuracy implies population accuracy. We formally introduce this notion of in-expectation accuracy for empirical counts. We introduce both its additive-only and relative counterparts, as they are both needed for our results.

\begin{definition}[In-Expectation Empirical Accuracy]
\label{def:expected_sample_accuracy}
Let $\cF\subseteq\{0,1\}^{\cX}$ and $\alpha\geq 0$.
We say that a randomized algorithm $\cA:\cX^n\mapsto\R^{\cF}$ is $(\alpha,\cF)$-sample-accurate in expectation if for all $D=(x_1,\ldots,x_n)\in \cX^n$,
\[ \mathbb{E}_{\cal A}\Big[ \sup_{f\in {\cal F}} \Big|\frac1n\sum_{i=1}^n f(x_i)-{\cal A}_f(D)\Big| \Big] \leq \alpha. \]
Let additionally $\zeta>0$. We say that $\cA$ is $(\zeta,\alpha,\cF)$-sample-accurate in expectation if for all $D=(x_1,\ldots,x_n)\in \cX^n$
\[ \mathbb{E}_{\cal A}\Big[ \sup_{f\in {\cal F}} \Big\{
{\cal A}_f(D)-(1+\zeta)\frac1n\sum_{i=1}^n f(x_i),
(1-\zeta)\frac1n\sum_{i=1}^n f(x_i)-{\cal A}_f(D)\Big\} \Big] \leq \alpha. \]
\end{definition}

\begin{proposition} \label{propos:generalization_synthetic_data}
Let ${\cal A}:{\cal X}^n\mapsto \R^{\cal F}$ be an algorithm  that is $(\alpha,\cF)$-sample-accurate in expectation.
Then ${\cal A}:{\cal X}^n\mapsto \R^{\cal F}$ is $(\alpha+2\mbox{Rad}(\cF),\cF)$-population-accurate in expectation, namely
\[ \sup_{\mathbb{P}} \mathbb{E}_{{\cal A},D\sim \mathbb{P}^n}\Big[ \sup_{f\in {\cal F}} \Big|\mathbb{E}_{x\sim \mathbb{P}}[f(x)]-{\cal A}_f(D)\Big| \Big] \leq 2\mbox{Rad}_n({\cal F})+\alpha,  \]
where $\mbox{Rad}_n({\cal F})\leq \sqrt{2\ln(|{\cal F}|)/n}$ is the expected Rademacher complexity of ${\cal F}.$
\end{proposition}

\begin{proof}
Let $\mathbb{P}$ be any distribution. By the triangle inequality,
\begin{align*}
&\mathbb{E}_{{\cal A},D\sim \mathbb{P}^n}\Big[ \sup_{f\in {\cal F}} \Big|\mathbb{E}_{x\sim \mathbb{P}}[f(x)]-{\cal A}_f(D)\Big| \Big]  \\
&\leq 
\mathbb{E}_{D\sim \mathbb{P}^n}\Big[ \sup_{f\in {\cal F}} \Big|\frac1n\sum_{i=1}^n f(x_i)-\mathbb{E}_{x\sim \mathbb{P}}[f(x)]\Big| \Big] 
+ \mathbb{E}_{{\cal A},D\sim \mathbb{P}^n}\Big[ \sup_{f\in {\cal F}} \Big|\frac1n\sum_{i=1}^n f(x_i)-{\cal A}_f(D)\Big| \Big]\\
&\leq 2\mbox{Rad}_n({\cal F})+\alpha,
\end{align*}
where the last step uses the sample accuracy of ${\cal A}$, and a classical symmetrization argument (e.g.~\cite{Mohri:2018}).
\end{proof}

\subsection{Proof of \Cref{thm:approx_DP_LB}.}

\label{sec:pf_approx_DP_LB}

We now use the previous results to prove the approximate DP lower bound.

\begin{proofof}{\Cref{thm:approx_DP_LB}}
First, we consider the case $\zeta=\bar\zeta:=\alpha/n$. Suppose that there exists a query-answering  $(\eps,\delta)$-DP algorithm ${\cal B}$ that is $(\bar\zeta,\alpha,\cF)$-sample-accurate in expectation for arbitrary sets $\cF$ of $k$ statistical queries. 
Let ${\cal B}$ be the algorithm that first runs ${\cal A}$, and then scales down the histogram counts by a factor $1/n$. Then algorithm ${\cal B}$ is $(\eps,\delta)$-DP (by postprocessing) and $(2\alpha/n,\cF)$-sample-accurate in expectation. Further, $2\alpha/n+2(1+\bar\zeta)\sqrt{2\ln(k)/n}\leq 10\alpha/n$, hence by \Cref{propos:generalization_synthetic_data}, ${\cal B}$ is also $(10\alpha/n,\cF)$-population-accurate in expectation.  
From \Cref{thm:LB_additive_query}\footnote{Note that this lower bound is applicable, as $2^{\ln^{-1/9}|{\cX}|}<10\alpha/n<1$ is implied by our assumption on $\alpha$.}, we conclude that 
\[ n=\Omega\Big(\frac{\ln|{\cF}|\sqrt{\ln|{\cX}| \ln(1/\delta)}}{\eps (10\alpha/n)^2} \Big).\]
This lower bound implies the two inequalities in \eqref{eqn:LB_approx_DP} with $\zeta=\bar\zeta$.

Next, consider the case $\zeta > \overline\zeta$.  Let $n'\leq n$ to be determined, and consider the histogram $\bh^{\ast}$ with $\|\bh^{\ast}\|_1=n$ comprised of $n'$ datapoints as in the lower bound construction from  
\cite{Lyu:2024} (call this histogram $\bh'$), padded with $n-n'$ dummy datapoints\footnote{More precisely, extend the sample space ${\cX}$ with an `empty' element $\dagger$, so that the dummy datapoints do not add counts to any queries.}. A $(\zeta,\alpha,\cF)$-sample-accuracy in expectation guarantee w.r.t.~$\bh^{\ast}$ implies $(\zeta,\alpha/n',\cF)$-sample-accuracy in expectation w.r.t.~$\bh'$. Letting $n'=\alpha/\zeta$, we reduce to the previous case, where we have concluded that
\begin{align*}
n \geq n'=\Omega\Big( \frac{\ln|{\cF}|\sqrt{\ln|{\cX}| \ln(1/\delta)}}{\eps \zeta^2}  \Big) \quad \mbox{and}\quad
\alpha = \Omega\Big( \frac{\ln|{\cF}|\sqrt{\ln|{\cX}| \ln(1/\delta)}}{\eps \zeta}  \Big). 
\end{align*}
Finally, when $\delta=0$ we can use the fact that $\eps$-DP algorithms are $(\eps,\delta)$-DP for any $\delta\geq 0$, and instantiating \Cref{eqn:LB_approx_DP} with $\delta=1/|\cX|$ yields the claimed bound.
\end{proofof}

\end{document}